\documentclass[10pt,twocolumn,letterpaper]{article}

\usepackage{iccv}              

\usepackage{mathtools}
\usepackage{multirow}
\usepackage{selectp}
\usepackage{tikz}
\usetikzlibrary{arrows.meta, positioning, bending}

\newlength{\bsize}

\def\relu{\phi}

\def\E{\mathbb{E}}

\def\R{\mathbb{R}}
\def\X{\mathcal{X}}
\def\I{\mathcal{I}}
\def\F{\mathcal{F}}
\def\S{\operatorname{S}}
\def\hf{\widehat{f}}
\def\hk{\widehat{k}}
\def\tf{\tilde{f}}

\def\EF{\operatorname{EF}}
\usepackage{amsthm}

\theoremstyle{plain} 
\newtheorem{theorem}{Theorem}
\newtheorem{lemma}{Lemma}
\newtheorem{definition}{Definition}
\newtheorem{remark}{Remark}
 
\definecolor{iccvblue}{rgb}{0.21,0.49,0.74}
\usepackage[pagebackref,breaklinks,colorlinks,allcolors=iccvblue]{hyperref}

\def\paperID{6397} 
\def\confName{ICCV}
\def\confYear{2025}

\title{
On the Complexity-Faithfulness Trade-off of Gradient-Based Explanations
}

\author{
Amir Mehrpanah \quad Matteo Gamba \quad Kevin Smith \quad Hossein Azizpour \\
KTH Royal Institute of Technology, Sweden\\
{\tt\small \{amirme, mgamba, ksmith, azizpour\}@kth.se}
}

\begin{document}
\maketitle

\begin{abstract}
ReLU networks, while prevalent for visual data, have sharp transitions, sometimes relying on individual pixels for predictions, making vanilla gradient-based explanations noisy and difficult to interpret. 
Existing methods, such as GradCAM, smooth these explanations by producing surrogate models at the cost of faithfulness. 
We introduce a unifying spectral framework to systematically analyze and quantify smoothness, faithfulness, and their trade-off in explanations.
Using this framework, we quantify and regularize the contribution of ReLU networks to high-frequency information, providing a principled approach to identifying this trade-off. 
Our analysis characterizes how surrogate-based smoothing distorts explanations, leading to an ``explanation gap'' that we formally define and measure for different post-hoc methods.
Finally, we validate our theoretical findings across different design choices, datasets, and ablations.\footnote{\href{https://github.com/Amir-Mehrpanah/On-the-Complexity-Faithfulness-Trade-off-of-Gradient-Based-Explanations-ICCV25/tree/main}{\texttt{https://github.com/Amir-Mehrpanah/On-the-\\Complexity-Faithfulness-Trade-off-of-Gradient-\\Based-Explanations-ICCV25/}}}
\end{abstract}
   
\section{Introduction}
The growing complexity of deep networks necessitates explainability, especially in safety-critical applications. 
Gradient-based explanation methods are widely used in computer vision for this aim. 
These methods, however, suffer from two main shortcomings: lack of \textit{smoothness} or \textit{faithfulness}, the trade-off of which is the focus of our study.

\noindent\textbf{Explanation complexity.} A crucial property for explanations to be human-understandable is their simplicity, which translates to smoothness\footnote{We may refer to smoothness and complexity interchangeably.} in saliency-based explanations. 
However, both theoretical~\cite{Gamba_2023_BMVC,jacot_neural_2020} and empirical~\cite{rosca_case_2021,smilkov_smoothgrad_2017} studies suggest that ReLU networks, the most common variant of deep networks used for visual data, have sharp transitions, sometimes relying on individual pixels for predictions~\cite{su_one_2019}.
This sharpness can result in complex gradient-based explanations that hinder interpretability. 
Thus, as our preliminary step, we establish a formal connection between a network's architecture and complexity of its explanations. With this formalism, we particularly study ReLU networks.

\noindent\textbf{Explanation faithfulness.} To reduce the explanation complexity, post-hoc techniques produce smooth surrogates during the explanation phase~\cite{mehrpanah_spectral_2025}, at the cost of lower faithfulness to the original network, leading to an inherent trade-off between explanation complexity and faithfulness. 
Therefore, it is vital to carefully consider this trade-off when developing, choosing, or using post-hoc methods~\cite{han_which_2022,adebayo_sanity_2020}. While several metrics are proposed to measure complexity (\eg, entropy-based methods~\cite{bhatt_evaluating_2020}) and faithfulness (\eg, pixel removal scores~\cite{bhatt_evaluating_2020}), these metrics disjointly measure one of the two aspects and are influenced by extraneous factors such as the choice of baseline and removal order~\cite{wickstrom_flexibility_2024,kindermans_reliability_2017}. These caveats of existing metrics impede a principled analysis of the complexity-faithfulness trade-off~\cite{ribeiro_how_2024,chen_true_2020}. 
Our goal is to formally study this trade-off. 
Specifically, we introduce consistently-defined measures of complexity and faithfulness with which we empirically analyze the existing post-hoc explanation methods, and our method of mitigation of sharpness in ReLU networks. 

\noindent\textbf{Our Approach.} We adopt a \textit{spectral} perspective to formally establish a relationship between a network's architecture and its gradient-based explanation. 
We first measure explanation complexity via spatial high-frequency content of input gradients (\cref{sec:complexity-to-tsps}), and formally link it to the high-frequency components, or sharp changes, of a network (\cref{sec:from-tps-to-tsps}).
This connection, importantly, allows for characterization and improvement of network architectures for their explanation complexity, which we showcase in~\cref{sec:controling-the-tps}. 
Finally, in~\cref{sec:gap}, we leverage our formalism in a unified framework to analyze the complexity-faithfulness trade-off, \ie, by measuring complexity as ``expected frequency'' and faithfulness as the ``gap'' in the expected frequency of the original and surrogate functions. 
We use this framework to analyze (smoothed versions of) ReLU networks and various post-hoc explanations; see~\cref{fig:banner}.  

\setlength{\bsize}{0.159\textwidth}
\begin{figure*}[t]
\centering
\begin{subfigure}[t]{\bsize}
    \centering
    \includegraphics[width=\linewidth]{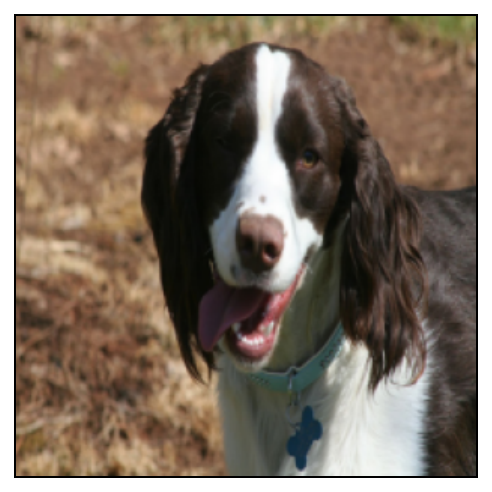}
    \caption{Input}
\end{subfigure}
\begin{subfigure}[t]{\bsize}
    \centering
    \includegraphics[width=\linewidth]{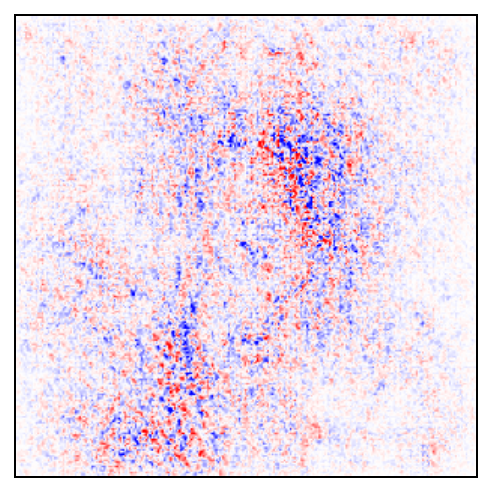}
    \caption{ReLU:VG}
    \label{fig:banner-b}
\end{subfigure} 
\begin{subfigure}[t]{\bsize}
    \centering
    \includegraphics[width=\linewidth]{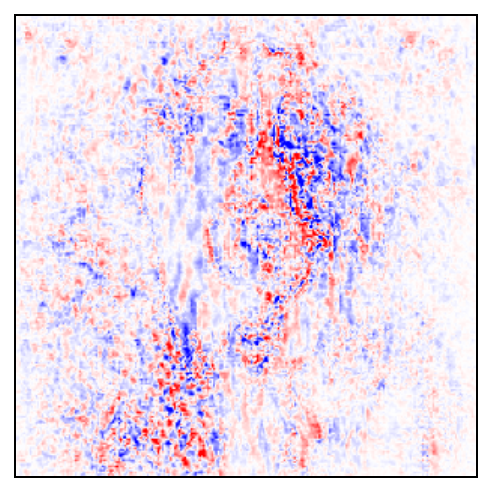}
    \caption{ReLU:SG}
\end{subfigure} 
\begin{subfigure}[t]{\bsize}
    \centering
    \includegraphics[width=\linewidth]{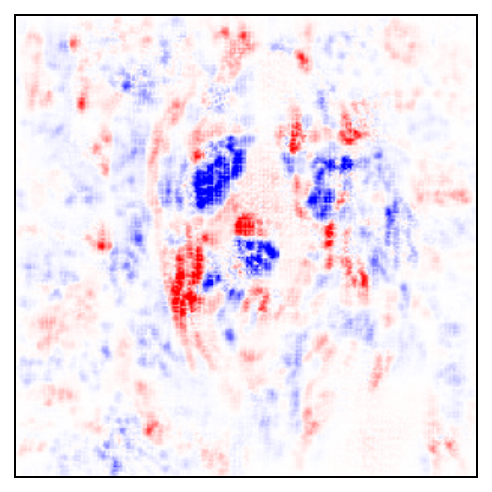}
    \caption{$\operatorname{SP}(\beta=.9)$:VG}
\end{subfigure}
\begin{subfigure}[t]{\bsize}
    \centering
    \includegraphics[width=\linewidth]{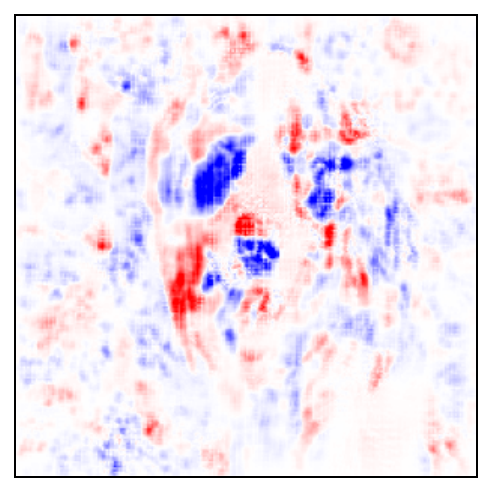}
    \caption{$\operatorname{SP}(\beta=.9)$:SG}
\end{subfigure} 
\begin{subfigure}[t]{1.11\bsize}
    \vspace{-2.8cm}
    \centering
    \includegraphics[width=\linewidth,trim=0mm 3.1mm 0mm 0mm, clip=true]{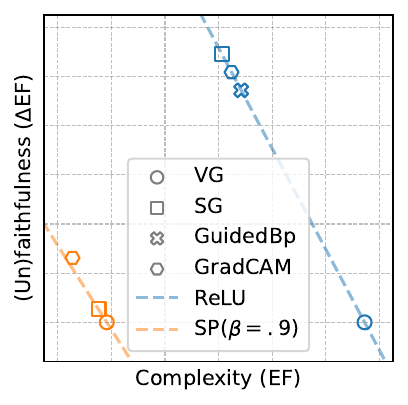}
    \caption{}
    \label{fig:banner-tails}
\end{subfigure}
   \caption{
   \textbf{Exploring the Faithfulness-Complexity Trade-Off.}
    This figure shows the trade-off typically made between faithfulness and complexity in post-hoc explanation methods.  
    \textbf{(b)} presents VanillaGrad (VG) explanations for a ReLU network, which appear grainy due to the network's reliance on high-frequency information. To mitigate this, post-hoc methods like SmoothGrad (SG), shown in \textbf{(c)}, remove high-frequency components by creating surrogates, thereby introducing a trade-off between complexity and faithfulness, quantified by \(\EF\) and \(\Delta\EF\), respectively. This trade-off for a ReLU network is represented by the dashed blue line in \textbf{(f)}.
    In this work, we propose a framework which can control a ReLU network's reliance on high-frequency information by modifying the tail of its power spectrum. This is achieved by training a network with a smooth parameterization (SP) of ReLU, obtained by convolving ReLU with a Gaussian function, where the standard ReLU corresponds to \(\beta \to \infty\).  
    \textbf{(d)} and \textbf{(e)} show the VG and SG explanations for such a network, denoted as \(\operatorname{SP}(\beta = 0.9)\). As shown, reducing the network's dependence on high-frequency components leads to VG explanations with lower complexity, enabling alternative trade-offs. One such trade-off is depicted by the dashed orange line in \textbf{(f)}—see~\cref{fig:more-examples} for more examples.
   }
\label{fig:banner}
\end{figure*}

\noindent\textbf{Our Contributions.} 
We summarize in the following list:
\begin{itemize}
    \item 
    \textbf{A Framework for Formal Analysis:} We introduce a framework based on a formal connection between a network’s power spectrum tail and the complexity of its explanations, 
    which provides a foundation for analyzing the complexity-faithfulness trade-off. 
    Using this framework, we analyze the role of ReLU in inducing sharp transitions in the network, leading to complex explanations. 

    \item 
    \textbf{The Expected Frequency (EF):} We introduce a metric to \textit{quantify} a network’s reliance on high-frequency information. We reduce EF by convolving ReLU with a Gaussian function, hence \textit{controlling} the complexity of VanillaGrad explanations (\cref{sec:ef}).
    
    \item
    \textbf{The Explanation Gap:} We introduce the \textit{``explanation gap''}, a measure for quantifying the discrepancy between explanations produced for surrogates and the original model, highlighting potential faithfulness risks associated with post-hoc explanation methods, in~\cref{sec:gap}. 
    
    \item
    \textbf{Empirical Validation:}
    We empirically validate our theoretical contributions across various datasets of varying input sizes, such as ImageNet, Imagenette, CIFAR10, and Fashion-MNIST. Moreover, we conduct several ablation studies, in~\cref{sec:other-decisions}, including the effect of other design decisions such as skip connections and batch norm.

\end{itemize}

\section{Related Works}
\label{sec:related-works}
Although inherent and post-hoc explainability are distinct concepts~\cite{rudin_stop_2019}, both can rely on gradients as their explainer, as seen in methods like B-cos~\cite{bohle_b-cos_2022} and Integrated Gradients~\cite{sundararajan_axiomatic_2017}, respectively.
We focus on methods that utilize gradients as their explainer, without making a distinction between post-hoc or inherent explainability.

\subsection{Explanation for Smooth Surrogates}
Our work is related to research on explainability methods that incorporate smooth surrogates, either as part of their design or in post-processing. Whether explicit or implicit, a common part in these methods is the use of smooth surrogates to improve visual quality.

One approach to creating smooth surrogates is through a \textit{perturbation distribution}, which can involve noise in the input space~\cite{smilkov_smoothgrad_2017}, noise in the parameter space~\cite{bykov_noisegrad_2022}, input space traversal~\cite{kapishnikov_guided_2021,schrouff_best_2022}, spatial Gaussian blur~\cite{bohle_b-cos_2022}, or removing negative gradient values~\cite{kim_why_2019,springenberg_striving_2015}.

Another method for creating surrogates is by using \textit{subnetworks}, as in GradCAM~\cite{selvaraju_grad-cam_2020}, or by making \textit{modifications} to the original structure, as seen in LRP~\cite{bach_pixel-wise_2015}, DeepLIFT~\cite{shrikumar_not_2017}, and their variants.

Throughout our work, we base our formal analysis on the simplest gradient-based method, VanillaGrad~\cite{simonyan_deep_2014}, which does not involve any form of smooth surrogate.

\subsection{Spectral View of Explanation}
Closely related to ours is the research that adopts a spectral perspective on explanations~\cite{kolek_cartoon_2022,kolek_explaining_2023,tsuzuku_structural_2019}, where the Fourier basis is defined along the spatial dimensions, such as width and height. While these methods typically use optimization to impose certain spectral properties on the explanations, our work leverages the spectral view solely for measurements in spatial domain.
Also, our work is connected to spectral view to networks' learnt functions~\cite{kadkhodaie_generalization_2024,marchetti_harmonics_2024,rahaman_spectral_2019}.

\subsection{Metrics for Evaluating Explanations}
Evaluating explanations remains challenging, with quite a few quantitative techniques available. 
Prior work has assessed explanations through human studies~\cite{lage_human_2019} and auxiliary tasks to verify attribution relevance~\cite{yang_benchmarking_2019}, while others test whether explanation features are meaningful to the model~\cite{camburu_can_2019}. 
A common approach, called pixel removal score~\cite{bhatt_evaluating_2020}, generates out-of-distribution samples, prompting a computationally heavy remove-and-retrain strategy~\cite{hooker_benchmark_2019}. 
Moreover, the correlation between output and pixel removal score is used as faithfulness metric~\cite{bhatt_evaluating_2020}.
However, despite a high correlation, explanations for surrogate and original models can differ in high-frequency components~\cite{tsuzuku_structural_2019}. 
Unlike these metrics, our formalism provides hyperparameter-free metrics with clearer intuition that better capture feature interactions.
\section{Model Structure: A Root Cause for Explanation Complexity}
\label{sec:ef}
As widely discussed in the literature, there is a recognized need to refine raw gradient-based explanations, commonly known as VanillaGrad, which are sometimes referred to as ``noisy''~\cite{bykov_noisegrad_2022,smilkov_smoothgrad_2017}.  
This stems from the human preference for simpler and more interpretable explanations.  

Input gradients of deep networks often appear scattered and noisy~\cref{fig:banner-b}. However, an important question arises: \textit{Are we truly observing noise in VanillaGrad explanations, or do these high-frequency variations stem from an underlying structural property of the network?}  


In the literature, this noisiness is commonly referred to as ``Explanation Complexity’’. In this section, we introduce a formalism to characterize a root cause for gradient-based explanations’ complexity: \emph{
the network’s architecture, which influences its explanation complexity}. 

After introducing our notation in~\cref{sec:complexity-to-tsps}, we
measure explanation complexity by a summary statistic of (high) spatial frequency content of input gradients. 
This is followed by establishing a formal connection between the high spatial frequency of input gradients and high frequency contents (or sharpness) of a network’s function in~\cref{sec:from-tps-to-tsps}. 
This central connection provides interesting implications for understanding and reducing explanation complexity. 

In~\cref{sec:controling-the-tps}, we reveal one such important implication for the most common family of deep networks in the visual domain, \ie ReLU networks. 
Finally, in~\cref{sec:gap}, we leverage our formalism for
a principled analysis of the explanation complexity and its trade-off with faithfulness to the model--two
axes to evaluate explanation methods that are measured in a unified framework, for the first time, in this work. 

\subsection{Notations}
To maintain generality and simplicity, we consider a scalar multivariate classifier denoted by $f:\R^n\rightarrow\R$. 
In the context of gradient-based explainability, we are particularly interested in computing the gradient of the network $f$ trained on a dataset $\X$ with respect to an input sample $x$, \ie $x\mapsto\nabla_xf$.  
For brevity, we may omit the $x$-subscript on the gradient operator, as gradients are consistently taken with respect to the input throughout this article.

Building on prior work that adopts a spectral perspective on neural networks~\cite{rahaman_spectral_2019}, we introduce a compact notation, 
to represent the Fourier basis. Under general assumptions, such as absolute integrability ($\int|f(x)|dx<\infty$), a function $f(x)$ can be decomposed into its frequency components: 
$\hf(\omega)=\int f(x)\psi(\omega,x) dx$,
yielding an alternative representation of the function in the frequency domain, which we denote by $\hf(\omega)$.
To analyze the frequency content of the explanations, we will also use the power spectrum (PS) of the signals, defined as the squared magnitude of the Fourier transform, $\S_f(\omega)=|\hf(\omega)|^2$.

A subtle technical detail arises in defining the Fourier basis along the spatial dimensions (orthogonal to the gradient direction) and along the pixel values (parallel to the gradient direction). We will focus only on the tail behavior of two power spectra in our analysis: 1) the Tail of Spatial Power Spectrum \textit{(TSPS)} of input-gradient of a network, where spatial means that the Fourier basis are defined along the spatial dimensions, \ie width and height of the input-gradient. 2) the Tail of Power Spectrum \textit{(TPS)} of a network, where the Fourier bases are defined along the pixel values.

\subsection{Explanation Complexity and TSPS of Gradient}
\label{sec:complexity-to-tsps}
\begin{figure*}[t]
\centering
\begin{subfigure}[t]{0.2465\linewidth}
    \centering
    \includegraphics[width=\linewidth]{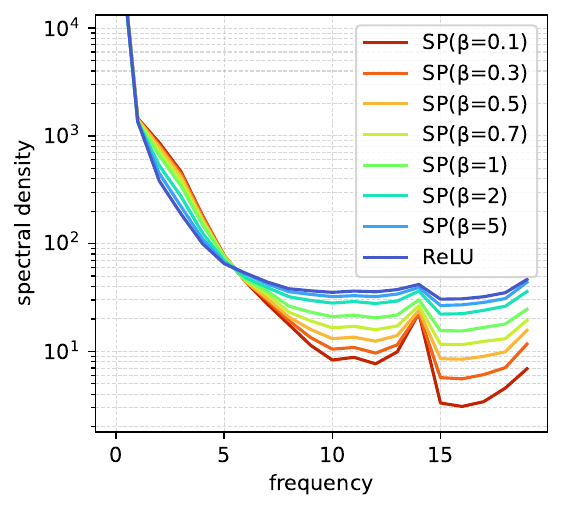}
    \caption{}
\end{subfigure}
\begin{subfigure}[t]{0.2465\linewidth}
    \centering
    \includegraphics[width=\linewidth]{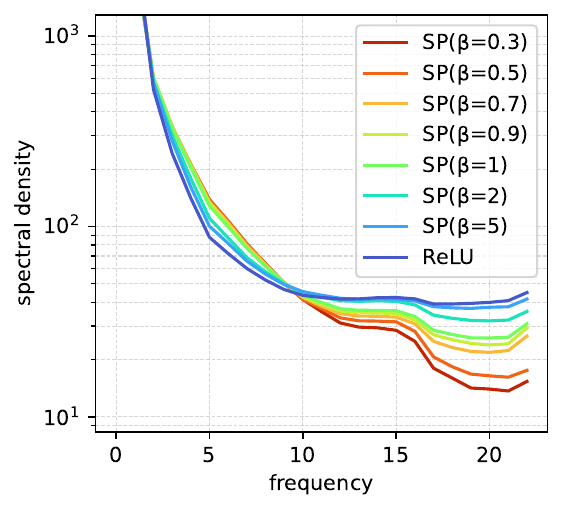}
    \caption{}
\end{subfigure}
\begin{subfigure}[t]{0.2465\linewidth}
    \centering
    \includegraphics[width=\linewidth]{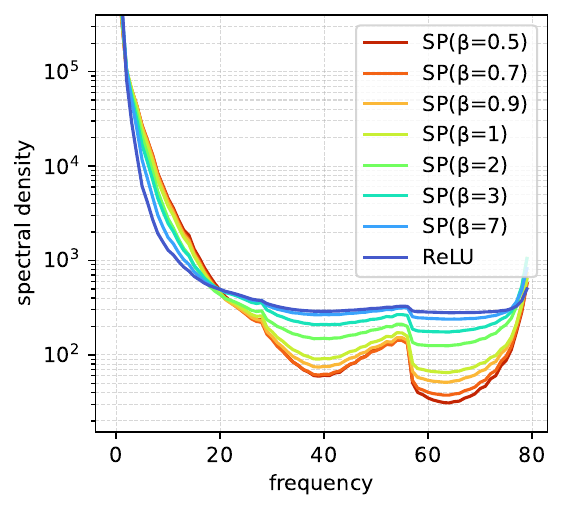}
    \caption{}
\end{subfigure}
\begin{subfigure}[t]{0.2465\linewidth}
    \centering
    \includegraphics[width=\linewidth]{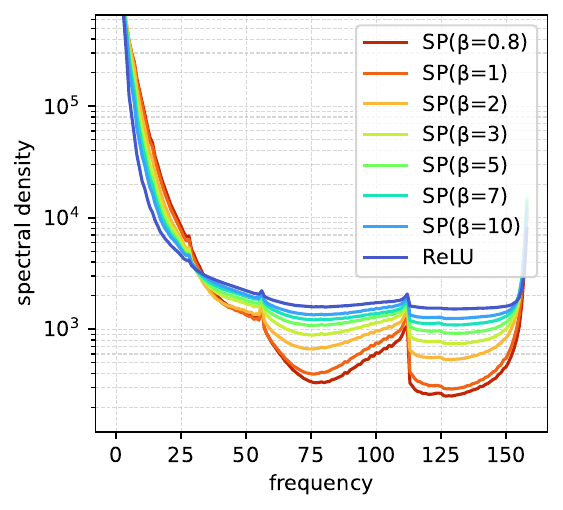}
    \caption{}
\end{subfigure}
   \caption{
   \textbf{ReLU Affects Gradient TSPS via Network's TPS.}
   This figure presents a practical implication of our theory on the tail behavior of the spatial power spectrum in relation to function sharpness, as formalized in Lemma~\ref{lemma:1}.
   The plots depict the TSPS of input-gradients of networks trained with \textit{identical} architectures and hyperparameters, varying \textit{only} in smoothness parameter of ReLU (indicated in the legend) and datasets: \textbf{(a)} Fashion MNIST, \textbf{(b)} CIFAR10, \textbf{(c)} Imagenette (112$\times$122), and \textbf{(d)} Imagenette (224$\times$224). Refer to~\cref{tab:hparams} in the Appendix for details on the hyperparameters used.
   In all plots, reducing smoothness by increasing $\beta$ (shown as SP($\beta$) in the legends) results in power spectra with heavier tails, leading to more complex explanations. Since larger complexity necessitates removing more of the tail, this contributes to a larger explanation gap, defined in~\cref{eq:delta-ef}.
   Considering that not all frequencies are equally relevant for classification, notable spikes in mid-range frequencies are particularly evident in \textbf{(a)}, \textbf{(c)}, and \textbf{(d)}. Additionally, some spectra exhibit a jump at the extreme end of the tail (\ie, very high frequencies). We hypothesize that this is due to the model’s reliance on edges of the objects, which are less distinct in CIFAR10 but more pronounced in Fashion MNIST, explaining the absence of a sharp tail in \textbf{(c)} and its rise in other cases.}
\label{fig:psd-tail-inputsizes}
\end{figure*}

The noisier the explanation, the heavier the tail of the spatial power spectrum, a trend that can also be observed empirically in~\cref{fig:psd-tail-inputsizes}.  
Thus, a natural way to quantify the complexity of explanations is through a summary statistic that captures the tail behavior of the spatial power spectrum of the input gradient.  

To this end, we introduce \textit{Expected Frequency}, a simple yet effective summary statistic defined as:
\begin{equation}
 \EF(e_f(x))\coloneq
 \int \omega \S_{e_f(x)}(\omega)d\omega.  
 \label{eq:ef}
\end{equation}
Here, $\S$ represents the spatial power spectrum of the explanation method $e_f:\R^n\rightarrow\R^n$ applied to the function $f$, which may simply be VanillaGrad. Hereafter, we may drop the subscript $f$ from $e_f$ for notational simplicity.
As we focus on image data, we use the radial average in frequency domain to have a 1D power spectrum $\S(\omega)$.


Since the concept of high-frequency information is inherently relative, 
it can be measured in various ways.
We found that expected frequency provides an effective statistic of model’s reliance on high-frequency features.
Nonetheless, the use of expectation is not pivotal for the results and is solely based on simplicity and practical effectiveness.
The utility of EF is emphasized in the following remark:

\begin{remark}
According to~\cref{eq:ef}, Expected Frequency is influenced by both the model and the explanation method used. A lower EF value corresponds to smoother, or simpler, explanations in the spatial domain.
\end{remark}

We hypothetize that expected frequency is influenced by neural networks' dependence on high-frequency components for making predictions. 
It is crucial to note that EF captures the tail behavior of spatial power spectrum of the gradient. 
In contrast, a network's utilization of high-frequency information is quantified by the tail of its power spectrum.
In the next section, we establish a connection between the tail spatial power spectrum of the gradient and the tail of power spectrum of the network in data modalities with high input feature correlation.

\subsection{TPS of a Network and TSPS of Gradient}
\label{sec:from-tps-to-tsps}
Having established a measure for explanation complexity in the spatial domain, we now take steps toward addressing the question: \textit{What causes the input gradient to look noisy?}

Our goal is to support an intuitive yet challenging-to-formalize claim: the presence of noise in the input gradient reflects the network's reliance on high-frequency information for its predictions. Ideally, if we had direct access to the network's Fourier transform, we could compute its power spectrum and directly compare it with the input gradient spectrum to verify this claim. However, since the Fourier transform of arbitrary networks is inaccessible, we instead have to look at the network’s power spectrum through the aperture of the spatial power spectrum of its gradient.

This idea is captured in the following theorem, which establishes a relationship between the tails of two power spectra: the tail of the spatial power spectrum of the input gradient and the tail of the power spectrum of the network itself. For clarity and narrative coherence, we defer the proof and technical details to~\cref{sec:proof-technical}.

\begin{figure*}[t]
\centering
\begin{subfigure}[t]{0.2465\linewidth}
    \centering
    \includegraphics[width=\linewidth]{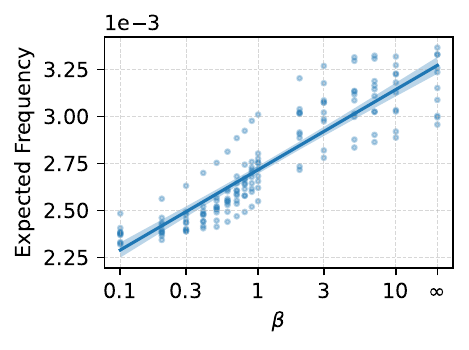}
    \caption{}
\end{subfigure}
\begin{subfigure}[t]{0.2465\linewidth}
    \centering
    \includegraphics[width=\linewidth]{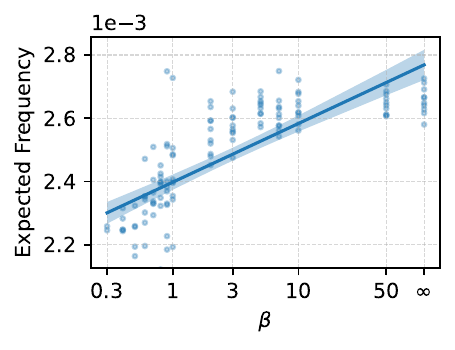}
    \caption{}
\end{subfigure}
\begin{subfigure}[t]{0.2465\linewidth}
    \centering
    \includegraphics[width=\linewidth]{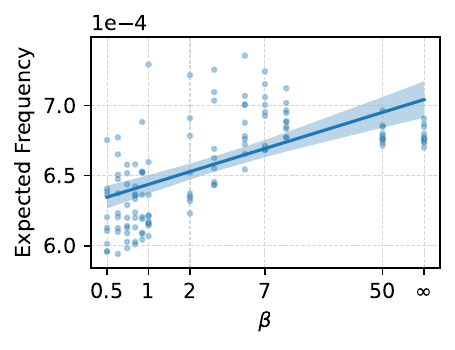}
    \caption{}
\end{subfigure}
\begin{subfigure}[t]{0.2465\linewidth}
    \centering
    \includegraphics[width=\linewidth]{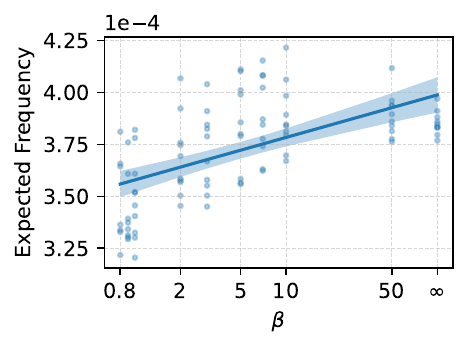}
    \caption{}
\end{subfigure}
   \caption{
   \textbf{Impact of Smoothness Parameter on Expected Frequency.}
   This figure shows EF, defined in~\cref{eq:ef}, used as a summary statistic of the frequency content of explanations in the spatial domain with respect to the smoothness parameter $\beta$ in a (log) x-axis. 
   To obtain the $95\%$ CI, we repeated experiments using 10 seeds for each dataset:  
   \textbf{(a)} Fashion MNIST, \textbf{(b)} CIFAR-10, \textbf{(c)} Imagenette ($122\times122$), and \textbf{(d)} Imagenette ($224\times224$).
   The results consistently show that, on average, EF in the spatial domain increases with the smoothness parameter $\beta$, where $\beta\rightarrow\infty$ corresponds to a ReLU network. 
   The variance observed in the scatter plots, we believe, is due to the weight initialization.}
\label{fig:ef-inputsizes}
\end{figure*}
\begin{theorem}[informal]
In data domains with high input feature correlation, \eg image data; given a trained neural network $f(x)$, the tail behavior of the power spectrum of $f(x)$, is directly proportional to that of the spatial power spectrum of $\nabla f(x)$, denoted by $\S_{\nabla f}(\omega)$.
\label{thm:1}
\end{theorem}
\begin{proof}[Proof sketch.]
We note that the decay rate of the tail of the power spectrum of a function is influenced by the sharp changes in the function created by the network. 
We show that the decay rate depends on the existence of such sharp transitions rather than the number of their occurrences.
Therefore, we consider a single training sample and a single explained sample with at least one sharp change. 
By applying a linear approximation around this sharp change in the spatial domain, we make the spatial Fourier transform analytically tractable; see~\cref{fig:kernel-shape-and-sharpness,fig:low-feature-autocorrelation} for a conceptual visualization, and~\cref{sec:proof-technical} for technical details.
\end{proof}

In simple terms, assuming a high correlation between input features in the spatial domain (such as images), a heavier tail in the network's power spectrum leads to a heavier tail in the spatial power spectrum. Intuitively, a heavier tail indicates higher explanation complexity, leading to a higher expected frequency, as shown in~\cref{fig:banner} (see~\cref{sec:theoretically-unproven-assumptions}).

A key aspect of \cref{thm:1} is the emphasis on the ``tail''. This terminology signifies that the high-frequency characteristics of the network—specifically, its tail behavior—are mirrored in the tail of the spatial power spectrum of the gradient. However, this correspondence may not accurately reflect the network's behavior at or near zero frequency.


In this section, we formalized an intuitive statement in ~\cref{thm:1}.
Next, we explore the practical implications of this theory, which also serve as an empirical validation of our insights.
Specifically, we tweak the sharp transitions introduced by ReLU in the network and test whether our theoretical predictions align with empirical observations.

\subsection{On the TPS of ReLU Networks}
\label{sec:controling-the-tps}
Having established the connection between the TPS of the network and the TSPS of its input gradient, we now turn to the first practical implication of our theory. Specifically, in this section, we control the TPS of a ReLU network and observe its impact on the TSPS of the input gradient.
 
Given its widespread use in modern architectures, and the growing evidence that ReLU induces sharp transitions in neural networks~\cite{Gamba_2023_BMVC,jacot_neural_2020,rosca_case_2021,smilkov_smoothgrad_2017}, we focus on the effect of ReLU for studying its effect on the TPS of a network. 
Nonetheless, there are many other design decisions, see~\cref{sec:other-decisions}, that can be studied in this regard.

To this end, we introduce the following lemma, which provides a simple approach to reduce the sharpness of transitions in the final trained function, thereby leading to a faster-decaying tail of the power spectrum of the network.

\begin{lemma}[informal]
Let $\xi$ be a smoothed version of an activation function $\phi$, achieved by convolving it with a Gaussian function with precision $\beta$, \ie $\xi = \phi*g_{\beta}$.
Let $f_{\phi}(x)$ and $f_{\xi}(x)$ denote the classifiers trained on $\X$ using activation functions $\phi$ and $\xi$, respectively.
Then $f_{\phi}$, exhibits a heavier tail in its power spectrum compared to $f_{\xi}$.
\label{lemma:1}
\end{lemma}
\begin{proof}[Proof sketch.] Building on prior work in kernel methods, we show that convolving ReLU with a Gaussian results in a smoother kernel without explicit characterization of the kernel. See~\cref{sec:relu-vs-Softplus} for formal details.
\end{proof}

To summarize, using the above lemma, we can control the TPS of the network, and through~\cref{thm:1}, we expect to observe its effect on the TSPS of the input gradient. 
Specifically, we employ a smooth parameterization of ReLU and gradually interpolate from a smoothed version toward the standard ReLU. 
This corresponds to increasing the precision parameter $\beta$ of the Gaussian convolution, where ReLU is recovered in the limit $\beta\rightarrow\infty$. 
This process effectively interpolates between networks with smoother boundaries—exhibiting faster decay rates in their power spectrum—toward ReLU-like networks characterized by sharp transitions and heavier tails.

Based on our theory, we anticipate that as $\beta$ increases, the TSPS of the input gradient will exhibit heavier tails, as illustrated in~\cref{fig:psd-tail-inputsizes}. 
This trend should also manifest in our measure of explanation complexity, namely, the expected frequency. 
Since a lower EF corresponds to simpler explanations, we expect smoother explanations as $\beta$ decreases, which can also be seen in~\cref{fig:banner}. 
The results in~\cref{fig:ef-inputsizes} confirm this, consistently showing that EF in the spatial domain increases with the smoothness parameter $\beta$ of ReLU.

We would like to highlight that~\cref{lemma:1} is being leveraged in two ways: (1) to validate our theoretical prediction regarding the effect of the TPS of the network on the TSPS of the input gradient and (2) to demonstrate a practical application of our theory beyond merely quantifying explanation complexity—namely, controlling it. 
Additionally, \cref{lemma:1} has a straightforward proof and is empirically verified in\cref{sec:relu-vs-Softplus}.
Please refer to~\cref{sec:implementation-details} for implementation details and hyperparamters.

As a more general application of this framework, in the next section, we introduce measures for faithfulness and examine its trade-off to complexity as measured by EF.

\section{The Complexity-Faithfulness Trade-off}
\label{sec:gap}
Surrogate models used for explanations aim to reduce complexity, often at the cost of faithfulness. This trade-off raises concerns, as the disparity between explanations generated by the original and surrogate models can be arbitrarily large~\cite{su_one_2019}. These methods rely on heuristically chosen hyperparameters, resulting in implicit surrogates with unknown properties, including their performance. This discrepancy can be significant enough to generate the same explanations for different model behaviors.

Thus, it is crucial to formally define the explanation gap in order to bring the faithfulness-complexity trade-off to the attention of researchers in XAI. We introduce a unifying spectral framework quantifying both aspects systematically.

\subsection{The Explanation Gap}
Let $\tf$ denote a (typically implicit) surrogate model for a function $f$, generated by an explanation method $e$. We define the \textit{explanation gap} as the squared $L^2$-norm of the difference between the gradients of $f$ and $\tf$:
\begin{equation}
    \mathcal{G}_{\X}(f,\tf)= \int_{x\in\X}\|\nabla f(x)-\nabla\tf(x)\|_2^2dx
    \label{eq:def-g-time-domain}
\end{equation}
Here, the subscript $\X$ represents the training data used to learn $f$, emphasizing the dataset dependence of the measure. This dependence means that explanation gaps are \textit{not} directly comparable across different datasets without appropriate normalization. For brevity, we may omit $\X$, though its influence remains implicit.

Notably, this measure exhibits two scaling behaviors: (1) it scales with the input dimension, and (2) it indirectly depends on the explanation method used to generate the surrogate $\tf$. 
As a result, different post-hoc explanation methods yield different explanation gaps.

\begin{remark}
An immediate property of this definition is that when using \textit{VanillaGrad}, which does not create a surrogate model (\ie, $\tf = f$), we obtain $\mathcal{G} = 0$.
\end{remark}

Since explanation methods are typically designed and refined through trial and error rather than with the explicit goal of constructing smooth surrogate models, the resulting surrogates are often implicit and inaccessible. This makes direct measurement of the explanation gap challenging, necessitating the identification of a suitable proxy measure.

In~\cref{sec:gap-in-fourier-domain}, we express the explanation gap in the Fourier domain to establish a connection between the gap and the TPS of the network. This reformulation provides several benefits: (1) it offers deeper insights into the nature of the explanation gap, (2) it links the gap to our definition of expected frequency, and (3) it clarifies the impact of surrogates on the explanation gap.

To build the necessary intuition, we first briefly introduce explanation methods in the Fourier domain in~\cref{sec:surrogates-low-pass}. For a more detailed discussion, see~\cref{sec:explanations-are-low-pass} and~\cite{mehrpanah_spectral_2025}.

\subsection{Explanation Methods as Low-Pass Filters}
\label{sec:surrogates-low-pass}
Previous works have provided unifying definitions of explanation methods, including set-theoretic~\cite{covert_explaining_2022} and probabilistic~\cite{mehrpanah_spectral_2025} perspectives. We adopt the probabilistic view, as it is more suitable for a spectral analysis.



In this view, many gradient-based explanations can be written as:
\begin{equation}
e_f(x) = \E_{\tilde{x}\sim \mathcal{N}(x)}[\nabla f(\tilde{x})],
\label{eq:exp-def}
\end{equation}
where $\tilde{x} \sim \mathcal{N}(x)$, a method-specific perturbation distribution. This form unifies a range of methods (see~\cite{mehrpanah_spectral_2025}, App. C) and facilitates a spectral analysis of perturbations, particularly in the Fourier domain (see~\cref{sec:explanations-are-low-pass}).

To develop theoretical insights, we focus on VanillaGrad, corresponding to $\mathcal{N}(x) = \delta(x)$, \ie, no perturbation. As a pure gradient method, it reveals the network’s full spectral behavior without surrogate-induced filtering.


In the Fourier domain:
\begin{equation}
\nabla f(x) \propto \int \omega \widehat{f}(\omega), d\omega,
\end{equation}
highlighting that, in ReLU networks, high-frequency components (in TPS), dominate the explanation (in TSPS), often resulting in noisy and grainy visualizations~\cite{mehrpanah_spectral_2025}, particularly in ReLU networks.


By contrast, perturbation-based methods suppress high-frequency components, effectively acting as low-pass filters. This smoothing improves visual clarity but may compromise faithfulness by shifting explanations away from the original model.



In the next section, we build on this formalism to refine our definition of the explanation gap.

\subsection{The Explanation Gap in Fourier Domain}
\label{sec:gap-in-fourier-domain}
Given that perturbations in explanation methods, or the creation of surrogates in general, are to suppress the high-frequency behavior of $f$, we can refine the definition in~\cref{eq:def-g-time-domain} to better account for this phenomenon.
By applying Parseval's theorem—which states that the $L^2$-norm of a function is equal in both the time and Fourier domains—we can rewrite the gap as follows:
\begin{align}
    \mathcal{G}(f,\tf)&=\int_{x\in\X}\|\nabla f(x)-\nabla\tf(x)\|^2dx\nonumber\\
    &\propto\int_{\omega\in\R}\omega^2\|\hf(\omega)-\widehat{\tf}(\omega)\|^2d\omega
\end{align}
The above integral can be broken down into high and low frequencies for some threshold $\omega^*$. Since surrogates suppress high-frequency components, one expects the integral on low frequency components to vanish, \ie $\int_{\omega\in\mathcal{F}_{\text{low}}}\approx0$. Therefore, we can approximate this as:
\begin{equation}
    \mathcal{G}(f,\tf)\approx\int_{\omega\in\mathcal{F}_{\text{high}}}\omega^2\|\hf(\omega)-\widehat{\tf}(\omega)\|^2d\omega
\end{equation}
This formulation emphasizes that the explanation gap is largely influenced by the high-frequency components of the classifier—specifically, \textit{the model's reliance on high-frequency features}, as captured by the TPS of $f$. However, a familiar challenge arises once again: the network’s power spectrum is not directly accessible, making it difficult to analyze its tail behavior.

Recall from~\cref{sec:ef} that we introduced expected frequency as a means to indirectly infer properties of the network's TPS through the TSPS of the input gradient. By leveraging this connection, we can now close the loop and use variations in EF as a proxy for the explanation gap, as detailed in~\cref{sec:quantify-the-gap}.

\subsection{Empirical Evaluation of the Explanation Gap}
\label{sec:quantify-the-gap}
In the previous sections, we expressed the explanation gap in the Fourier domain, showing that it is primarily influenced by the TPS of the network. We also established that the TPS is reflected in the spatial power spectrum, $\S_{\nabla f}$, and introduced expected frequency as a measure of the network’s reliance on high-frequency information, summarizing the TSPS of the input gradient. In this section, we bring these connections together to quantify the explanation gap.

More formally, we can use the absolute change in $\EF$ caused by an explanation method as a proxy for explanation gap, expressed as:
\begin{equation}
    \mathcal{G}(f,\tf) \sim \Delta \EF(e_f)\coloneq |\EF(\nabla f) -\EF(e_f)|
    \label{eq:delta-ef}
\end{equation}
See \cref{sec:derivations-delta-ef} for derivations. 

\begin{remark}
Due to the inherent scaling properties of the gap, this quantity is primarily meaningful in a relative sense when comparing different explanation methods. Informally, it reflects the \textit{degree of engineering} involved in creating a surrogate to achieve lower-complexity explanations.

Thus, $\Delta\EF$ reveals a range of behaviors rather than a dichotomy between post-hoc and ante-hoc explainability.
\end{remark}

This definition closely resembles a direct difference in the pixel domain, at the same time offers advantages such as shift invariance (via the Fourier transform) and rotation invariance (through the power spectrum). It also allows for flexible normalization strategies in the Fourier domain. While this measure can be highly correlated with direct pixel-based differences, its foundation in spectral analysis provides a more formal and unifying perspective on explanation complexity and faithfulness.

As demonstrated in~\cite{mehrpanah_spectral_2025}, many post-hoc explanation methods function as low-pass filters, suppressing high-frequency model behavior through implicit surrogate models. 
Having introduced the explanation gap as a measure of faithfulness—supported theoretically in~\cref{thm:1}—we now use this proxy to quantify the gap introduced by various post-hoc explainability methods. 
However, this should not be interpreted as an endorsement of methods that suppress high-frequency components implicitly, as such alterations may distort the original model's behavior in ways that are unknown and difficult to quantify.

To quantify the explanation gap, we compute EF, defined in~\cref{eq:ef} and measure its change after introducing a surrogate, using this difference as a proxy measure, as in~\cref{eq:delta-ef}. Note that the smoothness can be seen even in single images (\cref{fig:more-examples}), however, in~\cref{tab:positive-gap-act,tab:positive-gap-arch} we report our EF measure average over 1K images. We observe stds to be
vanishingly small, and as such chose not to report it.

\begin{table}
\begin{center}
\begin{tabular}{l|l|c|c}
  & \multirow{2}{*}{Method} &\multicolumn{2}{|c}{$\downarrow\EF$ + $\downarrow\Delta \EF$}\\
  \cline{3-4}
  && ReLU & $\operatorname{SP}(\beta=.9)$\\
\hline\hline
\multirow{7}{*}{\rotatebox[origin=c]{90}{Imagenette--CNN}}& VanillaGrad~\cite{simonyan_deep_2014} & $.390+\Delta .000$ & $.202+\Delta .000$\\
& SmoothGrad~\cite{smilkov_smoothgrad_2017} & $.286+\Delta.104$ & $.196+\Delta.005$\\
& IntGrad~\cite{sundararajan_axiomatic_2017} & $.396+\Delta.007$ & $.205+\Delta.003$\\
& GuidedBp~\cite{springenberg_striving_2015} & $.300+\Delta.090$ &  $.202+\Delta.000$\\
& DeepLift~\cite{shrikumar_not_2017} & $.394+\Delta.005$ & $.204+\Delta.002$\\
& GradCAM~\cite{selvaraju_grad-cam_2020} & $.293+\Delta.097$ & $.177+\Delta.025$\\
& LRP~\cite{binder_layer-wise_2016} & $.394+\Delta.005$ & Undefined\\
\hline
\end{tabular}
\end{center}
\caption{
\textbf{The Continuous Range of Faithfulness-Complexity Trade-offs.}
This table shows the evaluation of expected frequency in~\cref{eq:ef} and explanation gap in~\cref{eq:delta-ef} on different post-hoc explanation methods, scaled by $10^4$. As can be seen in this table, one can achieve smaller expected frequency while still having zero explanation gap (VanillaGrad row). Thus, $\Delta\EF$ captures a continuous range of behaviors rather than enforcing a rigid distinction between post-hoc and ante-hoc explainability.}
\label{tab:positive-gap-act}
\end{table}

\begin{table}
\begin{center}
\begin{tabular}{l|l|c|c}
  & \multirow{2}{*}{Method} &\multicolumn{2}{|c}{$\downarrow\EF$ + $\downarrow\Delta \EF$}\\
  \cline{3-4}
  && ResNet50\cite{he_deep_2015} & ViT-B16\cite{wu_visual_2020}\\
\hline\hline
\multirow{7}{*}{\rotatebox[origin=c]{90}{ImageNet}}& VanillaGrad~\cite{simonyan_deep_2014} & $.263+\Delta .000$ & $.222+\Delta .000$\\
& SmoothGrad~\cite{smilkov_smoothgrad_2017} & $.247+\Delta.017$ & $.221+\Delta.001$\\
& IntGrad~\cite{sundararajan_axiomatic_2017} & $.253+\Delta.010$ & $.221+\Delta.001$\\
& GuidedBp~\cite{springenberg_striving_2015} & $.294+\Delta.031$ &  $.222+\Delta.000$\\
& DeepLift~\cite{shrikumar_not_2017} & $.254+\Delta.009$ & $.224+\Delta.003$\\
& GradCAM~\cite{selvaraju_grad-cam_2020} & $.133+\Delta.130$ & $.181+\Delta.041$\\
& LRP~\cite{binder_layer-wise_2016} & $.282+\Delta.019$ & Undefined\\
\hline
\end{tabular}
\end{center}
\caption{
\textbf{The Faithfulness-Complexity Trade-off in Other Architectures.}
This table presents the evaluation of the expected frequency from~\cref{eq:ef} and the explanation gap from~\cref{eq:delta-ef} across various post-hoc explanation methods for two pretrained architectures on ImageNet, scaled by $10^4$. As shown, GradCAM, a widely used method known for generating smoother explanations, consistently introduces the largest explanation gap across different architectures. In contrast, most classical saliency methods, originally designed for feedforward convolutional networks, exhibit lower variability when applied to vision transformers.
Also the activation function used in VIT is GELU which exhibits a high empirical similarity with SP (but an empirical comparison between RELU and GELU ViT, shows that the ViT architecture contributes to lower variability more than the activation used see \cref{tab:relu-vs-gelu-vit}).
}
\label{tab:positive-gap-arch}
\end{table}


Thus, the explanation gap is computed by first determining EF for both VanillaGrad and a given explanation method, then taking the difference. Furthermore, to demonstrate the broader applicability of our approach, we measure these quantities across different architectures, even in cases where our theoretical results may not strictly apply. The results are presented in~\cref{tab:positive-gap-act,tab:positive-gap-arch}.
\section{Implementation Details and Ablations}
\label{sec:implementation-details}
To accommodate both sharp and smooth ReLU variants within a Smooth Parameterization (SP), we used SoftPlus, an efficient approximation:
\begin{equation}
    \text{SoftPlus}(x;\beta) = \frac{1}{\beta} \ln(1+e^{\beta x})\approx \text{ReLU}*g_\beta(x)
\end{equation}
where $\beta$ controls smoothness, corresponding to the precision of the Gaussian convolution $g_\beta$.

To isolate ReLU’s effects, we used feedforward convolutional networks, avoiding interference of confounding factors like residual connections and batch normalization (\cref{sec:other-decisions}). Experiments were conducted across four datasets with varying input sizes (\cref{fig:psd-tail-inputsizes}).



Since ReLU networks typically converge faster than smooth parameterizations, we set dataset-specific validation accuracy caps for early stopping, preventing initialization biases (see~\cref{sec:other-decisions} for an ablation and~\cref{fig:cap-ablation-acc} for a comparison of validation accuracies achieved with each parameterization). This ensured meaningful comparisons under similar training budgets (see~\cref{tab:hparams} for hparams).


For ImageNet~\cite{deng_imagenet_2009}, pretrained networks were used to measure frequency and explanation gaps (\cref{tab:positive-gap-act,tab:positive-gap-arch}), leveraging the Captum library~\cite{kokhlikyan_captum_2020}.


We also conducted ablation studies on network depth and learning rate (\cref{sec:other-decisions}). Depth had little effect on spectral decay rates, while learning rate influenced curve shapes without altering overall tail behavior (\cref{fig:psd-tail-inputsizes}).


Notably, differences in the spatial power spectra of various models are robust and often observable from a single image. However, for improved reliability, we averaged decay rates across 1K images, using a consistent batch for all activations and depths within each dataset.
Furthermore, we have used normalized $\S$ in~\cref{eq:ef}, to have a distribution over frequencies, yet the values are comparable without normalization within the same architecture and dataset.

Given that Fourier analysis depends on the distribution of gradient magnitudes, we employed the inverse transformation method to normalize rankings for each pixel~\cite{greitans_spectral_2005,stoica_spectral_2005}. This ensured that results were independent of absolute gradient magnitudes, aligning with the intuition that explanation magnitudes can sometimes be uninformative and are better interpreted through rankings (see~\cref{sec:ranking-why-not}).

\section{Limitations and Future Work}
\label{sec:limitations}
We conjecture that expanding this research into a unified theoretical framework to analyze various design decisions would necessitate a more advanced mathematical theory of deep networks, an area that is currently lacking in the field~\cite{petersen_mathematical_2024}. 
Nonetheless, an comprehensive empirical analysis of the design choices can pave the way for a more advanced and encompassing theory. 

Given that this work relies on kernel methods for the connection between the tail of the network's power spectrum and the tail of the spatial power spectrum of input-gradient, caution is advised, particularly where the kernel perspective becomes counterintuitive, as seen in the results on depth~\cite{bietti_deep_2021}, or the assumption of infinite depth.
Nonetheless, it can be seen empirically, particularly in our work, that the conclusions hold in finite width networks. 

Our analysis assumes continuity in the spatial domain for analytical convenience. A suitable discretization of our framework is necessary for each resolution for insights into the tail behavior of networks in practice.

We believe this work lays the foundation for a more formal analysis and a systematic investigation of potential approaches, such as neural architecture search using the spectrum’s tail as an optimization objective.

Additionally, it presents a novel perspective on the traditional ante-hoc vs. post-hoc explainability dilemma, reframing them as two orthogonal axes ($\EF+\Delta\EF$) with a continuum of possible behaviors.


\section{Conclusion}
\label{sec:conclusion}
In this work, we introduced a unifying spectral framework to systematically analyze the complexity-faithfulness trade-off in gradient-based explainability. By leveraging Expected Frequency (EF) as a principled metric, we quantified a network’s reliance on high-frequency information and established a connection between network behavior and explanation complexity. Our findings demonstrate that explanation complexity and faithfulness are deeply intertwined through spectral properties, with surrogate models often introducing a significant ``explanation gap.''

Through theoretical analysis and empirical validation, we demonstrated that controlling the spectral properties of ReLU networks can lead to smoother explanations without compromising faithfulness. 
Our results offer practical implications for designing both neural networks and explainability methods that maintain interpretability with higher faithfulness to the original model. This study lays the foundation, summarized in~\cref{fig:graphical-view}, for future research in more formal approaches to explainability of neural networks with respect to their inherent explainability properties.
\section*{Acknowledgements}
This project is partially supported by Region Stockholm through MedTechLabs, and  \href{https://wasp-sweden.org/}{Wallenberg AI, Autonomous Systems and Software Program (WASP)} funded by the Knut and Alice Wallenberg Foundation. Scientific computation was enabled by the supercomputing resource Berzelius, provided by the National Supercomputer Center at Linköping University and the Knut and Alice Wallenberg foundation.

{
\small
\bibliographystyle{ieeenat_fullname}
\bibliography{main,references}
}
\newpage
\quad\\
\newpage
\appendix
\section{Explanation Methods are Low-pass Filters}
\label{sec:explanations-are-low-pass}
This section builds on key insights from studies focused on the spectral analysis of gradient-based explanation methods, which provides a foundation for understanding how these methods interact with different frequency components.

Gradient-based explanation methods often~\cite{adebayo_sanity_2020, haruki_gradient_2019} incorporate a perturbation mechanism to reduce noise in raw gradients, as seen in VanillaGrad. This perturbation can be represented as a probability distribution $p(x)$ in the input space of the neural network. For mathematical clarity, we focus on raw gradients, though the spectral properties of squared gradients are also covered in \cite{mehrpanah_spectral_2025}.

The process of sampling and averaging can be formulated as an expectation over the perturbation distribution:
\begin{equation}
    \E_{\mathcal{N}(x)}[\nabla f(x)]
\end{equation}

\begin{figure*}
    \centering

\begin{tikzpicture}[
    node distance=2cm,
    >=Stealth,
    every node/.style={align=center},
    box/.style={draw, rounded corners, minimum width=2cm},
    edge label/.style={
        fill=white,
        inner sep=2pt,
        font=\footnotesize,
        anchor=center,
    }
]

  \node (perturbation) {perturbation};
  \node (nabla) [below=of perturbation] {input-gradient};
  \node (relu) [below=of nabla] {Relu};
  \node (tpsf) [right=of nabla, box] {TPS f};
  \node (tspsnablaf) [right=of tpsf, box] {TSPS $\nabla$ f};
  \node (ef) [below=of tspsnablaf, box] {EF};
  \node (complexity) [below=of ef] {Complexity};
  \node (deltaef) [right=of ef, box] {$\Delta$ EF};
  \node (gap) [below=of deltaef] {Gap};

  \draw[->, bend left=30] (perturbation) to node[edge label] {\cite{mehrpanah_spectral_2025}} (tpsf);
  \draw[->] (nabla) to node[edge label] {\cite{mehrpanah_spectral_2025}} (tpsf);
  \draw[->, bend right=30] (relu) to node[edge label] {\cref{sec:controling-the-tps}~\cref{lemma:1}} (tpsf);
  \draw[->, bend left=60] (tpsf) to node[edge label] {\cref{sec:from-tps-to-tsps}~\cref{thm:1}} (tspsnablaf);
  \draw[->] (tspsnablaf) to node[edge label] {\cref{eq:ef}} (ef);
  \draw[->] (ef) to node[edge label] {\cref{sec:complexity-to-tsps}} (complexity);
  \draw[->, bend left=60] (ef) to node[edge label] {\cref{eq:delta-ef}} (deltaef);
  \draw[->] (deltaef) to node[edge label] {\cref{sec:quantify-the-gap}} (gap);

\end{tikzpicture}
    \caption{
    \textbf{Overview of Key Theoretical Connections.}
    A graphical overview of our contributions, which is based on prior works on spectral analysis of gradient-based explanation methods. The diagram illustrates the conceptual flow of the narrative presented in this paper. From TPS $f$ to $\EF$ and $\Delta\EF$, highlighting key theoretical results and their related sections for analyzing complexity and explanation gap in a unified theoretical framework.
    A follow-up for this work is finding out how other architectural components affect TSP of a network.}
    \label{fig:graphical-view}
\end{figure*}
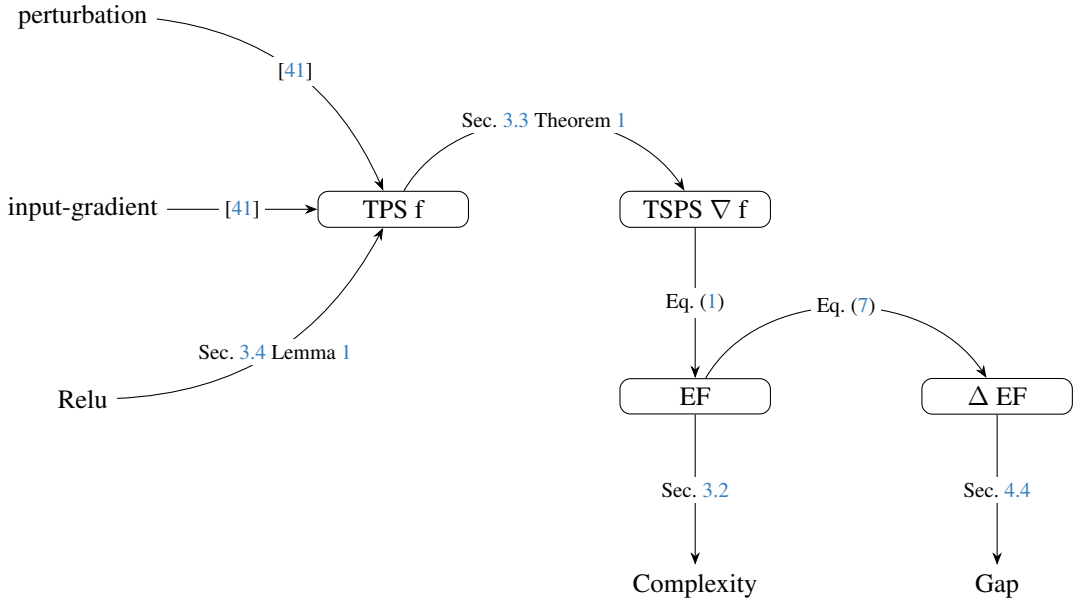

Using this formulation, we can derive a spectral representation of explanation methods: 
\begin{equation}
    \E_{\mathcal{N}(x)}[\nabla f(x)]\propto\int\omega \hf(\omega)\widehat{\mathcal{N}}(\omega)d\omega
\end{equation}
where $\hf(\omega)$ denotes the Fourier transform of the neural network, $\widehat{\mathcal{N}}(\omega)$ the Fourier transform of the perturbation distribution, and $\omega$ arises as a scaling factor due to the Fourier transform of the gradient.

This equation highlights the inherent filtering behavior of gradient-based explanation methods. The gradient operator acts as a high-pass filter, emphasizing high-frequency components of the model function, while the perturbation mechanism (\eg, Gaussian noise in SmoothGrad) serves as a low-pass filter, attenuating high-frequency components. The combined effect forms a band-pass filter, which selectively attributes importance to features within specific frequency ranges. This interplay between gradient computation and input perturbation fundamentally shapes the behavior of gradient-based explanations.

Given this behavior, we focus solely on the VanillaGrad in this work and disregard the variations in the neighborhoods, \ie assuming $\mathcal{N}(x)=\delta(x)$, in our theoretical analysis.
\section{An Empirical Study of Sharpness via Tail}
\label{sec:other-decisions}
In this work, we have primarily examined the impact of ReLU on the tail of the network's power spectrum in the main text, as it is a prevalent choice in many architectures used in contemporary computer vision. 
However, on the empirical side, our approach is not limited to this activation function, and similar analyses can be extended to other architectural choices. 
In this section, we empirically investigate the effects of other design decisions on the tail of the power spectrum.

\begin{figure}[t]
\centering
    \includegraphics[width=\linewidth]{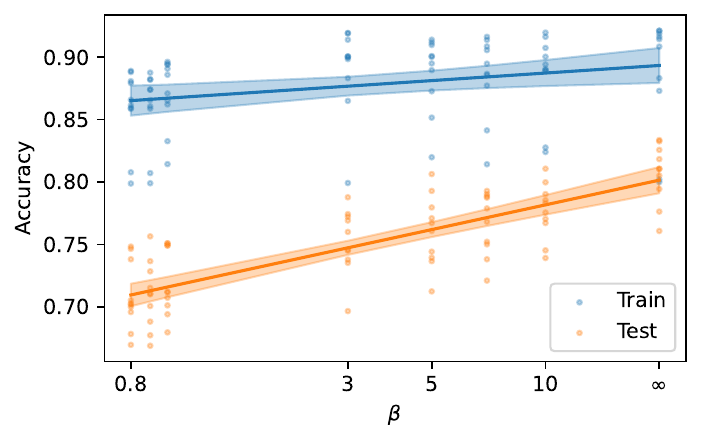}
   \caption{
   \textbf{Ablation Study: Impact of Smoothness Parameter on Validation Accuracy.}
   This figure presents an ablation study of our decision to impose an accuracy cap as an early stopping mechanism on Imagenette ($224 \times 224$). By relaxing this constraint, we train smooth parameterizations of ReLU networks with varying $\beta$ parameters (shown on the x-axis), where $\beta \rightarrow \infty$ corresponds to a standard ReLU network. As expected from the complexity-explainability tradeoff, restricting the network’s ability to learn high-frequency information results in a lower validation accuracy.}
\label{fig:cap-ablation-acc}
\end{figure}

To isolate the effect of the smooth parameterization of ReLU, we applied a validation accuracy cap to minimize the influence of initialization on our results. Since ReLU networks leverage well-established initialization schemes and generally achieve better convergence than our smooth parameterization, setting a high accuracy cap could introduce initialization as a confounding variable in our analysis. This approach allows for a meaningful comparison by ensuring networks are evaluated under similar training budgets.  

To investigate the impact of this validation accuracy cap, we conduct an ablation study where we remove the cap and train various networks with smooth ReLU parameterizations for approximately 200 epochs, continuing until their learning curves plateau.

The spatial power spectra for runs with different learning rates are shown in~\cref{fig:cap-ablation-tail}. Under these conditions, we can analyze the effect of the smoothness parameter $\beta$ on both training and test accuracy--see~\cref{fig:cap-ablation-acc}.

\begin{figure*}[t]
\centering
\begin{subfigure}[t]{0.24\linewidth}
    \centering
    \includegraphics[width=\linewidth]{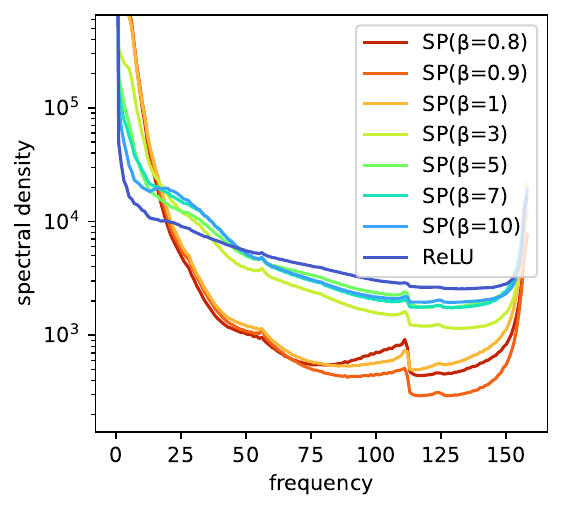}
    \caption{}
\end{subfigure}
\begin{subfigure}[t]{0.24\linewidth}
    \centering
    \includegraphics[width=\linewidth]{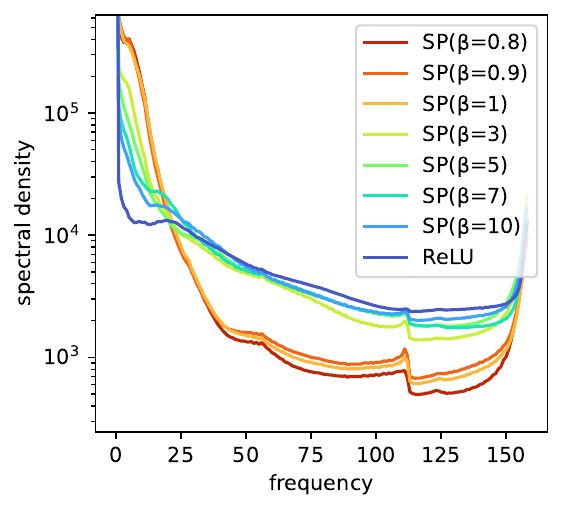}
    \caption{}
\end{subfigure}
\begin{subfigure}[t]{0.24\linewidth}
    \centering
    \includegraphics[width=\linewidth]{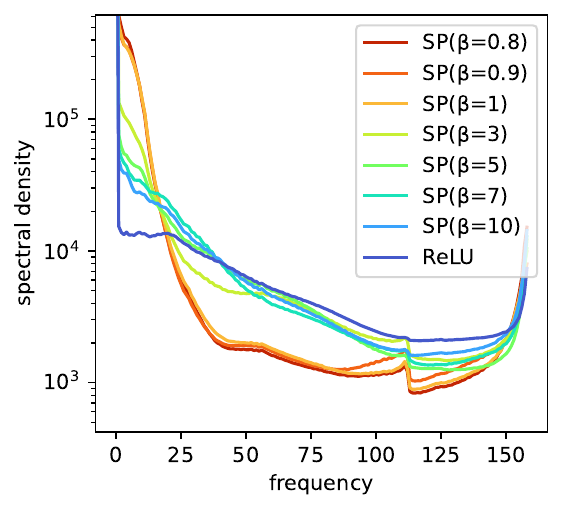}
    \caption{}
\end{subfigure}
\begin{subfigure}[t]{0.24\linewidth}
    \centering
    \includegraphics[width=\linewidth]{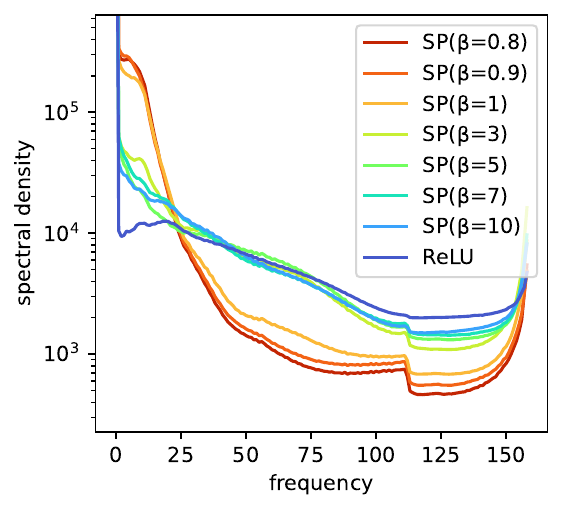}
    \caption{}
\end{subfigure}
   \caption{
   \textbf{Ablation Study: Impact of Validation Accuracy Cap and Learning Rate on the Spatial Power Spectrum.}
   This figure presents an ablation study on our decision to impose a validation accuracy cap as an early stopping mechanism on Imagenette ($224\times 224$) across different learning rates. To assess the impact of this choice, we train the networks for extended sessions of approximately 200 epochs without the accuracy cap. Important to note that, in this setting, the spatial power spectra of the functions do not correspond to networks with comparable functional behavior, as their performance levels differ--see~\cref{fig:cap-ablation-acc}. This discrepancy was the primary motivation for enforcing the accuracy cap.  
   Nevertheless, we observe that networks with higher smoothness parameter $\beta$, exhibit heavier tails in their spatial power spectra, indicating a greater tendency to learn higher-frequency information. However, this trend is less clear compared to the observations in~\cref{fig:psd-tail-inputsizes}.}
\label{fig:cap-ablation-tail}
\end{figure*}

Continuing our ablation study with learning rate and depth,~\cref{fig:psd-tail-depth} shows that the tail of the spatial power spectrum in a network with lower smoothness parameter $\beta$ contains less high-frequency content compared to a network with standard ReLU activations. This aligns well with earlier findings on NTK~\cite{bietti_deep_2021}, regarding the invariance of NTK with respect to depth.

\begin{figure*}
\begin{center}
\includegraphics[width=\linewidth]{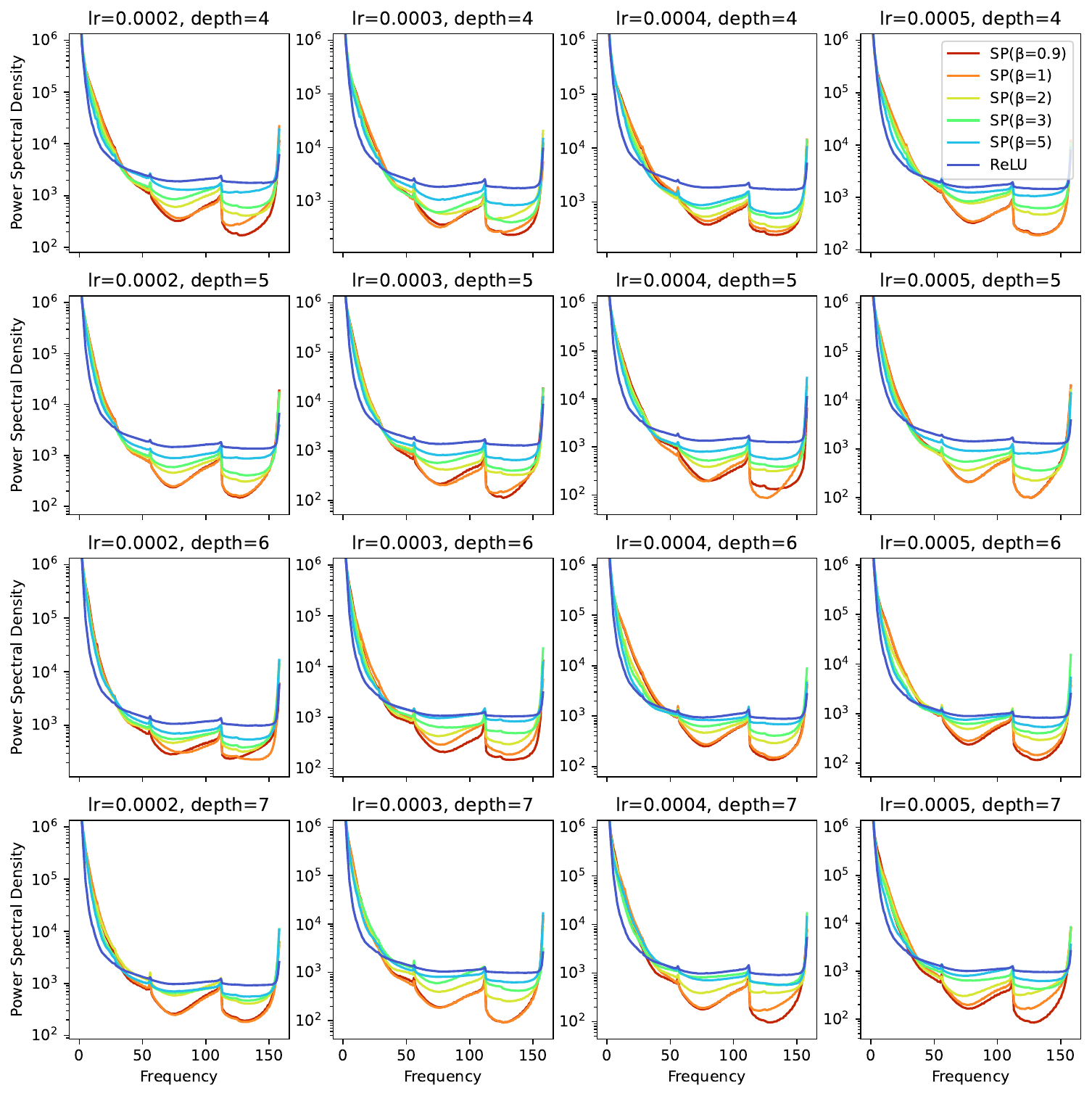}
\end{center}
   \caption{
   \textbf{Ablation Study: Impact of Depth and Learning Rate on the Spatial Power Spectrum.}
   This figure illustrates an ablation study on network depth (varied across rows) and learning rate (varied across columns) to assess their influence on the spatial power spectrum of the explanations. Empirically, the spatial power spectra remain largely unchanged with depth, aligning with earlier theoretical findings on NTK~\cite{bietti_deep_2021}. Additionally, the results support our theoretical framework, which predicts that increasing the smoothness parameter $\beta$, leads to heavier tails and, consequently, more complex explanations. We should note that the ReLU activation function is recovered whith $\beta\rightarrow\infty$, and all experiments were conducted on Imagenette $224\times224$.}
\label{fig:psd-tail-depth}
\end{figure*}

As previously discussed, input size plays a crucial role in our experiments and has been examined in prior studies by varying the dataset. To isolate the effect of input size alone, we trained models on different versions of the Imagenette dataset with input sizes of $224\times224$, $122\times122$, $64\times64$, and $46\times46$. The results are presented in ~\cref{fig:isize-ablation-tail}.

\begin{figure*}[t]
\centering
\begin{subfigure}[t]{0.24\linewidth}
    \centering
    \includegraphics[width=\linewidth]{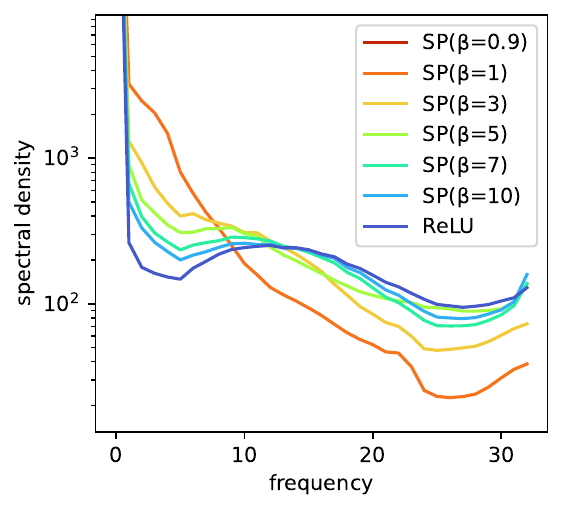}
    \caption{}
\end{subfigure}
\begin{subfigure}[t]{0.24\linewidth}
    \centering
    \includegraphics[width=\linewidth]{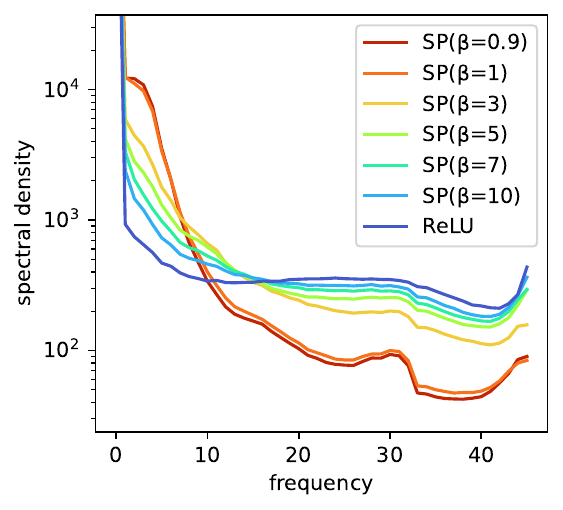}
    \caption{}
\end{subfigure}
\begin{subfigure}[t]{0.24\linewidth}
    \centering
    \includegraphics[width=\linewidth]{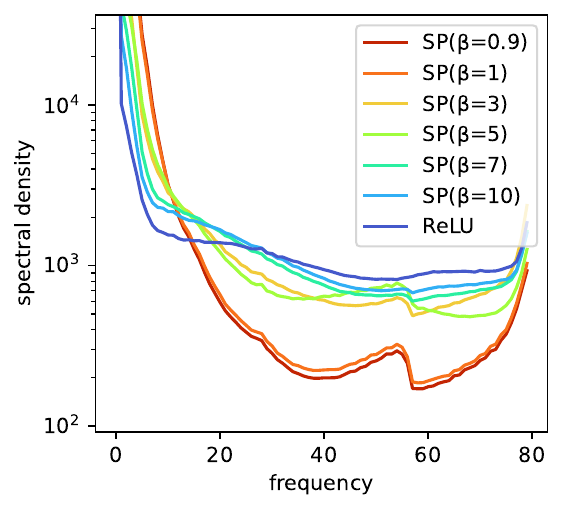}
    \caption{}
\end{subfigure}
\begin{subfigure}[t]{0.24\linewidth}
    \centering
    \includegraphics[width=\linewidth]{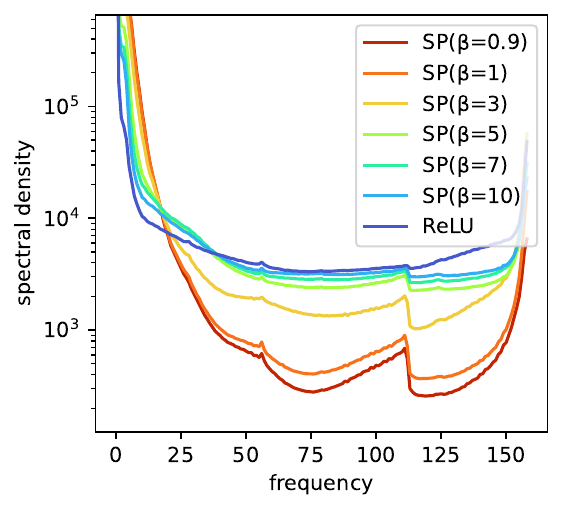}
    \caption{}
\end{subfigure}
   \caption{
   \textbf{Ablation Study: Impact of Input Size on the Spatial Power Spectrum.}
   This figure presents an ablation study, isolating the effect of input size by training multiple models on different versions of the Imagenette dataset with varying input resolutions. 
   The input sizes range from \textbf{(a)} 46$\times$46, \textbf{(b)} 64$\times$64, \textbf{(c)} 112$\times$112, and \textbf{(d)} 224$\times$224. 
   As can be observed, the overall trend remains consistent with previous findings in~\cref{fig:psd-tail-inputsizes}, though the peak in mid-range frequencies becomes less pronounced as the input size decreases. 
   This corroborates our conjecture about the correspondence of mid-frequency and very high frequency peaks to the reliance of the model on edges.}
\label{fig:isize-ablation-tail}
\end{figure*}

We have also examined the impact of skip connections and batch normalization on the tail of the spatial power spectrum, results shown in~\cref{fig:psd-tail-skbn}.
To be able to include skip connections, we have used slightly deeper networks.
While skip connections slightly affect the tail, batch normalization generally amplifies it significantly. 
Interestingly, this agrees with~\cite{cai_towards_2024} on the effect of batch normalization on learning high frequency information, yet suggesting a potential research direction to reconcile this observation with prior findings on skip connections mitigating gradient noise~\cite{balduzzi_shattered_2018}
In all cases, ReLU contributes to a heavier tail, whereas smoother versions reduce the expected frequency as defined in~\cref{eq:ef}.

\begin{figure}[t]
\centering
\begin{subfigure}[t]{0.7\linewidth}
    \centering
    \includegraphics[width=\linewidth]{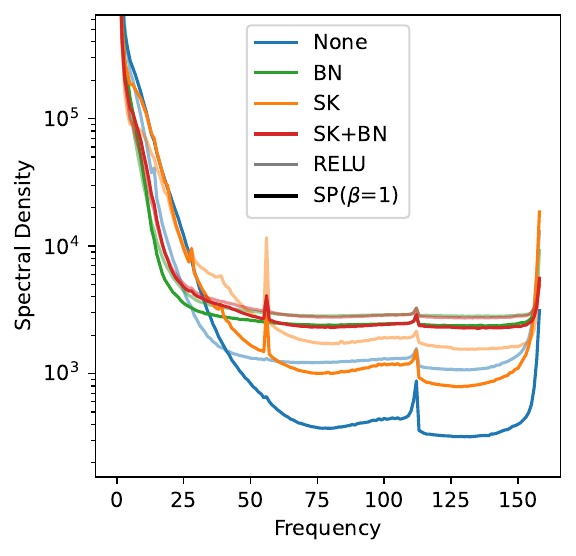}
    \caption{}
\end{subfigure}
\begin{subfigure}[t]{0.7\linewidth}
    \centering
    \includegraphics[width=\linewidth]{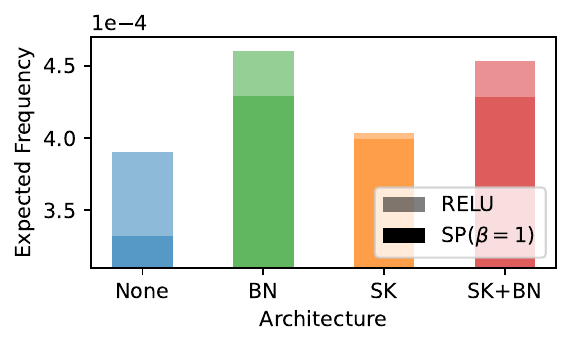}
    \caption{}
\end{subfigure}
   \caption{
   \textbf{Ablation Study: Effect of Skip Connections and Batch Normalization on the Spatial Power Spectrum Tail.}
   This figure illustrates the impact of skip connections and batch normalization on the tail of the spatial power spectrum on Imagenette ($224\times224$). In \textbf{(a)} the power spectrum's tail exhibits a slight increase in networks with skip connections. In contrast, batch normalization generally contributes to a heavier tail. Nonetheless, in all cases, ReLU tends to amplify the tail, whereas replacing it with a smoothed function with parameter $\beta=1$, reduces the expected frequency shown in \textbf{(b)}, as defined in~\cref{eq:ef}.}
\label{fig:psd-tail-skbn}
\end{figure}

We investigated the impact of input noise during training on the tail of the spatial power spectrum. Our observations indicate that Gaussian isotropic noise has a negligible effect, so we omitted the results to avoid redundancy.

\begin{table}
\begin{center}
\begin{tabular}{l|c|c|c|c}
\hline
Dataset & Input Size & LR & Depth & Cap\\
\hline\hline
Imagenette & 224$\times$224& $1e-4$& 5 & $60\%$ \\
Imagenette & 112$\times$112& $1e-4$& 5 & $60\%$\\
Imagenette & 64$\times$64& $3e-4$& 5 & $60\%$ \\
Imagenette & 46$\times$46& $5e-4$& 5 & $60\%$\\
CIFAR10 & 32$\times$32& $3e-3$& 4 & $70\%$\\
Fashion MNIST & 28$\times$28& $1e-4$& 3 & $80\%$\\
\hline
\end{tabular}
\end{center}
\caption{
\textbf{Table of hyperparameters.}
This table outlines the general hyperparameters used in our experiments analyzing the tail behavior of the power spectrum (TSPS) of gradient and its relation to that of the tail of the power spectrum (TPS) of the network. LR represents the learning rate, and Cap refers to an early stopping criterion based on validation accuracy. To isolate the effect of smooth parameterization of ReLU, we implemented the validation accuracy cap to reduce the impact of initialization on our findings. Since ReLU networks benefit from well-established initialization strategies and tend to exhibit better convergence properties compared to its smooth parameterization, a high accuracy cap could introduce initialization as a confounding factor to our analysis, see~\cref{sec:other-decisions} for an ablation study of this decision. This strategy ensures a meaningful comparison by comparing networks at similar training budgets.}
\label{tab:hparams}
\end{table}
\begin{table}
\begin{center}
\begin{tabular}{l|l|c|c}
  & \multirow{2}{*}{Method} &\multicolumn{2}{|c}{$\downarrow\EF$ + $\downarrow\Delta \EF$}\\
  \cline{3-4}
  && ReLU-ViT & GELU-ViT\\
\hline\hline
\multirow{7}{*}{\rotatebox[origin=c]{90}{Imagenette}}& VanillaGrad~\cite{simonyan_deep_2014} & $.239+\Delta .000$ & $.248+\Delta .000$\\
& SmoothGrad~\cite{smilkov_smoothgrad_2017} & $.239+\Delta.000$ & $.248+\Delta.000$\\
& IntGrad~\cite{sundararajan_axiomatic_2017} & $.244+\Delta.005$ & $.253+\Delta.005$\\
& GuidedBp~\cite{springenberg_striving_2015} & $.253+\Delta.000$ &  $.247+\Delta.014$\\
& DeepLift~\cite{shrikumar_not_2017} & $.245+\Delta.007$ & $.254+\Delta.006$\\
& GradCAM~\cite{selvaraju_grad-cam_2020} & $.205+\Delta.051$ & $.197+\Delta.033$\\
& LRP~\cite{binder_layer-wise_2016} & Undefined & Undefined\\
\hline
\end{tabular}
\end{center}
\caption{
\textbf{ReLU vs GELU in ViT-B16~\cite{wu_visual_2020}.}
This table reports the expected frequency (EF) from~\cref{eq:ef} and the explanation gap ($\Delta$EF) from~\cref{eq:delta-ef} across various post-hoc explanation methods for a ViT-B16 model trained from scratch on Imagenette, using different activation functions. All values are scaled by $10^4$.
The results indicate that the ViT architecture has a greater influence on lowering EF than the choice of activation function. 
Interestingly, the relative ordering of EF complexity between ReLU and GELU is inverted compared to theoretical expectations, which predict higher EF for ReLU. 
This discrepancy may stem from the fact that ViT induces a different kernel geometry~\cite{wright_transformers_2021} than the Laplace kernel assumed in our analysis~\cite{geifman_similarity_2020}.
Nonetheless, the activation function may still influence the smoothness (EF) and complexity ($\Delta$EF) of explanations.}
\label{tab:relu-vs-gelu-vit}
\end{table}

\section{A Short Introduction to Kernel Methods}
\label{sec:kernel-details}
A distinct line of research in machine learning aims to connect neural networks with the classical framework of kernel methods. This work introduced Neural Tangent Kernels (NTK)~\cite{jacot_neural_2020}, which help explain certain behaviors observed during neural network training. 
To ensure clarity, we briefly revisit the definition of kernels: a kernel is a symmetric function $k:\X\times\X\rightarrow\R$ such that the matrix $K_{ij}=k(x_i,x_j)$ is positive semidefinite for an arbitrary set $\X=\{x_1,\dots,x_n\}$. A kernel generally serves as a measure of similarity between two entities, such as inputs $x_i$ and $x_j$.

NTK provides a kernel-based perspective on neural networks, where similarity is defined in terms of weight gradients. Since our work in explainability relies on input gradients, we are particularly interested in the properties that NTK can reveal in this context. However, it is important to acknowledge that NTK was not originally designed for explainability, and purely theoretical predictions may be inaccurate. Therefore, experimental validation is crucial.

While a specialized version of NTK exists for convolutional networks—namely, the Convolutional Neural Tangent Kernel (CNTK)—we opted for a more general case, \ie NTK framework to support our observations. 
This choice is motivated by the fact that the reasoning in Lemma~\ref{lemma:1} relies on a core component of NTK, the $\tau$-transform. As demonstrated in our experiments, NTK provides a sufficiently accurate approximation for predicting the tail behavior of the power spectrum.

The discovery of NTKs represents a significant advancement in understanding neural networks. However, in this work, we primarily leverage results related to the sharpness of the kernel. Prior studies~\cite{geifman_similarity_2020} have identified a striking similarity between NTKs and the Laplace kernel, under certain technical conditions, see~\cite{geifman_similarity_2020}. Given that we are only interested in the tail behavior of the power spectrum, without loss of generality, we replace NTK with the Laplace kernel to simplify our analysis.

Our focus on the spectral decay properties of kernels induced by ReLU networks connects our work to studies on Reproducing Kernel Hilbert Spaces (RKHS), particularly NTK, the pre-activation tangent kernel (PTK), and related research~\cite{bietti_inductive_2019,daniely_toward_2017,geifman_similarity_2020,simon_reverse_2022,beaglehole_feature_2024} (see Appendices B and D in~\cite{beaglehole_feature_2024}). These connections suggest that insights in one domain may inform advancements in the other.

Additionally, our work is related to feature selection using kernels~\cite{allen_automatic_2013,chen_kernel_2018,gregorova_large-scale_2018,hou_kernel_2019}, highlighting the nuanced relationship between feature selection and explainability.

This section is based on key results from the kernel methods literature, particularly \cite{kanagawa_gaussian_2018}, which serves as a valuable resource for a deeper exploration. Kernel methods provide a powerful framework for non-parametric learning by implicitly mapping data into a high-dimensional feature space through a kernel function $k(x, x')$.

\setcounter{theorem}{0}
\begin{definition}
    A symmetric function $k:\X\times\X\rightarrow\R$, is called a positive semidefinite kernel, if the matrix $K_{ij}=k(x_i,x_j)$ is positive semidefinite for an arbitrary non-empty set $\X=\{x_1,\dots,x_n\}$.
\end{definition}

Common examples of such kernels include exponential functions of the form: $k(x, x') = \exp(-|x - x'|^\gamma)$, where the function is referred to as the Laplace kernel for $\gamma=1$ and the Gaussian kernel for $\gamma=2$.

A kernel, together with an associated inner product between functions, defines a Reproducing Kernel Hilbert Space (RKHS) $\mathcal{H}_k$, where functions inherit smoothness properties dictated by the choice of $k$. Notably, different kernels define RKHSs with varying smoothness constraints, and a key relationship between them is
\begin{equation}
    \mathcal{H}_{\text{Gaussian}}\subset \mathcal{H}_{\text{Laplace}}.
\end{equation}
This inclusion indicates that Gaussian RKHSs consist of smoother functions compared to those in the Laplace RKHS.

A fundamental result in kernel methods is the Representer theorem, which ensures that solutions to many learning problems can be expressed as kernel expansions: 
\begin{equation} 
f(x) = \sum_{i\in \I} \alpha_i k(x, x_i), 
\label{eq:representer}
\end{equation} 
where $x_i$ are training examples from the training set $\X$ indexed by $\I$, and $\alpha_i$ are learned coefficients.

A particularly important class of kernels, known as shift-invariant kernels, depends only on the absolute distance between inputs, \ie, $\Delta = \|x-x'\|$. 
With a slight abuse of notation, such kernels can be written as $k(\Delta)$. 
This property allows one to take the Fourier transform of the kernel with respect to $\Delta$, leading to a simplified yet insightful characterization of the RKHS:
\begin{equation} 
\mathcal{H}_k = \left\{ f : \int \frac{|\mathcal{F}\{f\}|^2}{\mathcal{F}\{k\}} d\omega < \infty \right\}. 
\end{equation} 
This expression relates the power spectrum of a function $f$ to the Fourier transform of the reproducing kernel $k$. Intuitively, a function $f$ belongs to the RKHS of $k$ if the decay rate of its power spectrum is at least as fast as that of $\mathcal{F}\{k\}$, thereby constraining the sharpness of functions representable by the kernel.

Recent developments in the kernel-based understanding of deep networks have led to the discovery of the Neural Tangent Kernel (NTK) \cite{bietti_inductive_2019}, which characterizes network behavior during training. 

The (empirical) NTK for a network $f$ with parameters $W^{\ell}$ at layer $\ell$ and two points $x_0$ and $z_0$ defined as follows,
\begin{equation}
    \hat{k}_{\ell}\left(x_0, z_0\right)=\left\langle\frac{\partial f\left(x_0\right)}{\partial W^{(\ell)}}, \frac{\partial f\left(z_0\right)}{\partial W^{(\ell)}}\right\rangle
\end{equation}
which is connected to pre-activation tangent kernel (PTK) $\mathcal{K}^{(\ell)}$ by noting that $\frac{\partial f\left(x_0\right)}{\partial W^{(\ell)}}=\frac{\partial f\left(x_0\right)}{\partial h_{\ell}} x_0^{\top}$. Therefore, we can write,
\begin{equation}
\hat{k}_{\ell}\left(x_0, z_0\right)=\mathcal{K}^{(\ell)}\left(x_0, z_0\right) \cdot x_{\ell}^{\top} z_{\ell} .    
\end{equation}
see Appendix B and D of~\cite{beaglehole_feature_2024} for details. 
Finally, since gradient-based explanation methods rely heavily on the network's input gradient, it is unsurprising that advancements in one domain can inform the other.

Notably, empirical and theoretical findings suggest that the NTK closely resembles the Laplace kernel~\cite{geifman_similarity_2020}, implying that the function space of neural networks is constrained similarly.
This insight provides a theoretical foundation for understanding the sharpness and expressivity of neural networks through a kernel lens.

Interestingly, despite being developed for different purposes, both NTK and gradient-based explanations rely on input gradients, suggesting that insights from one field can contribute to advancements in the other.

\section{CDF based Normalization}
\label{sec:ranking-why-not}

Our work is tangentially related to research on explainability, influenced by game-theoretic approaches to explanation~\cite{laberge_partial_2023,lundberg_unified_2017,wojtas_feature_2020}. We adopt the assumption that explanations can be represented as a ranking of input features.  

While the explainability community generally agrees that explanations can be expressed as rankings, the process of obtaining these rankings remains unclear. 
To derive rankings, existing pipelines incorporate various normalization strategies, ranging from simple techniques such as min-max normalization to more complex, heavily engineered approaches that are harder to reproduce.

Inspired by literature on spectral analysis of signals~\cite{greitans_spectral_2005,stoica_spectral_2005}, we assume that an explanation method gives rise to a distribution per image across pixels. 
More formally, if $x(i)$ denotes the random variable observed in each pixel of explanation, which is distributed according to $\pi(x(i))$, with $\Pi(x(i))$ being its corresponding cumulative distribution function.
To normalize the explanations, we use $\Pi(x(i))$ instead of the actual observed values, \ie $x(i)$.

To normalize this distribution within a comparable framework, we apply the inverse transformation method. This normalization technique aligns different distributions while preserving their spectral properties and remaining insensitive to magnitude, see~\cref{fig:vis-explainers-example}.

Compared to alternative normalization methods, such as norm or max normalization, our approach reveals clearer spectral-domain trends and is easier to reproduce, as it does not rely on extensive engineering.

While the inverse transformation method is useful in our setting, it is highly sensitive to small variations in the gradient. To mitigate this effect, we average the spectral densities over 1K samples, although meaningful results often emerge from a single image.

\section{Proofs and Technical Considerations}
\label{sec:proof-technical}
\begin{figure*}[t]
\centering
\begin{subfigure}[t]{0.24\linewidth}
    \centering
    \includegraphics[width=\linewidth]{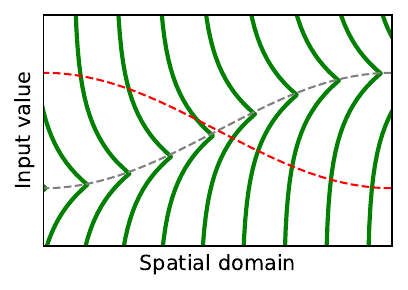}
    \caption{}
\end{subfigure}
\begin{subfigure}[t]{0.24\linewidth}
    \centering
    \includegraphics[width=\linewidth]{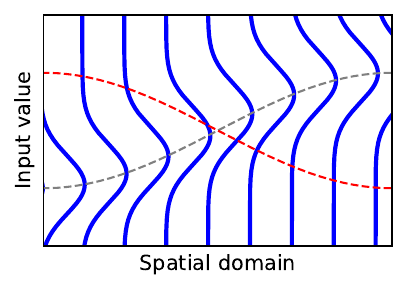}
    \caption{}
\end{subfigure}
\begin{subfigure}[t]{0.24\linewidth}
    \centering
    \includegraphics[width=\linewidth]{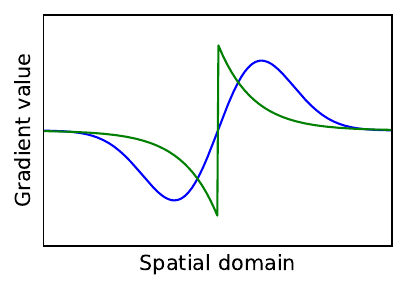}
    \caption{}
\end{subfigure}
\begin{subfigure}[t]{0.24\linewidth}
    \centering
    \includegraphics[width=\linewidth]{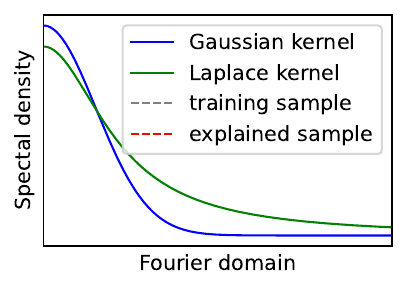}
    \caption{}
\end{subfigure}
   \caption{
   \textbf{Illustration of Kernel Sharpness on the Spatial Power Spectrum Tail.}
   A cartoon illustration depicts the theory regarding the effect of sharpness of the learned features on the tail of the power spectrum.
   \textbf{(a)} and \textbf{(b)}:
   The x-axis represents a simplified version of spatial dimensions (\eg, width or height), while the y-axis shows pixel values, consequently, each sample is visualized as a line. 
   Two samples are used for simplicity: the red line represents the sample to be explained, and the gray line represents a learned feature from the training data. 
   In (a), the features arise from the Laplace kernel (sharp), while in column (b), they arise from the Gaussian kernel (smooth).
   \textbf{(c)}:
   By taking the gradient of the classifier with respect to the input (along the y-axis in (a)) for the red line (the sample being explained), we get a function in the spatial dimension, which is visualized in (c).
   \textbf{(d)}:
   Applying the Fourier transform to the gradient values along the spatial dimension (x-axis of (c)) reveals different decay rates for the gradient functions, which are visualized in (d).
   This visualization highlights that the spectral properties of the gradient values for the sample being explained depend on the spectral properties of the kernel, as formalized in~\cref{thm:1}. 
   This visualization assumes high spatial autocorrelation between the learned features and the input, a characteristic typical of image data.
   For discussions on outcomes when this assumption is relaxed, refer to~\cref{sec:theoretically-unproven-assumptions}.
   While it is well established that the NTK's spectral properties closely resemble those of the Laplace kernel, here we use the Gaussian kernel purely as an illustrative example and do not explicitly characterize the kernel corresponding to a smoother variant of ReLU.}
\label{fig:kernel-shape-and-sharpness}
\end{figure*}

In this section, we emphasize that the tail behavior is a relatively stable property, meaning it does not change easily. While the proof involves a tedious case analysis to rule out various edge cases, it remains conceptually straightforward.

Let $k$ be a kernel equivalent to the Neural Tangent Kernel (NTK) of a network, and $\mathcal{I}$ is an index for our training set $\X$. We express our function in terms of its features as
\begin{equation}
    f(x)=\sum_{i\in\I}\alpha_i k(x,x_i)\nonumber
\end{equation}
where, same as~\cref{eq:representer}, $\X$ represents the training set.
Since, in explainability, we are interested in the spectral decay of the function’s input-gradient, we consider
\begin{equation}
\nabla_x f(x) = \sum_{i\in\I} \alpha_i\nabla_x k(x,x_i),
\end{equation}
however, we may drop the subscript $x$ from $\nabla_x$ as we only take the input-gradient.

\begin{lemma}
Let \(\mathcal{X} = \{x_i\}_{i=1}^{n} \subset \mathbb{R}^d\) be a dataset of size \(n\), and let \(k: \mathbb{R}^d \times \mathbb{R}^d \to \mathbb{R}\) be a shift-invariant kernel with Fourier transform \(\hat{k}(\omega)\). Define the input-gradient of the kernel as \(\nabla_x k(x, x')\). Then, the spectral decay of the input-gradient of \(k\) satisfies the bound:  
\[
\left|\mathcal{F}\{\nabla_x k\}\right|^2 = \mathcal{O}(n\,\omega^2 |\hat{k}(\omega)|^2),
\]
where \(\mathcal{F}\) denotes the Fourier transform, aligned with the direction of the gradient, and \(\omega\) is the frequency variable in the spectral domain.
    \label{lemma:2}
\end{lemma} 
\begin{proof}
To establish an upper bound on the tail behavior of the gradient of the kernel, we write
\begin{align}
    \left|\F\{\nabla f(x)\}\right|^2&=\left|\F\left\{\sum_{x_i\in\I} \alpha_i \nabla k(x,x_i)\right\}\right|^2\\
    \label{eq:order-sum-hk}&\propto\omega^2\left|\sum_{i\in\I} \alpha_i\F\{k(x,x_i)\}\right|^2\\
    &\leq\omega^2\sum_{i\in\I} \alpha_i^2\left|\F\{k(x,x_i)\}\right|^2\\
    &\leq\omega^2|\hk(\omega)|^2 \sum_{i\in\I} \alpha_i^2\\
    &\in\mathcal{O}\left(n\,\omega^2 |\hk(\omega)|^2\right)
    \label{eq:order-sample-size-hk}
\end{align}
\end{proof}

\begin{remark}
As can be seen, the bound $\mathcal{O}\left(n\omega^2 |\hk(\omega)|^2\right)$ for input-gradient of the kernel scales with the dataset size. Yet, we are not interested in the effect of dataset size on the problem and compare models trained on fixed datasets. Hence, For simplicity of exposition, hereafter we assume $n=1$.
\label{remark:n=1}
\end{remark}

According to~\cref{remark:n=1}, we expect the spectral decay of the kernel gradient to depend primarily on the Fourier transform of the kernel itself, specifically the $|\hk(\omega)|^2$ term in~\cref{eq:order-sample-size-hk}.

\begin{remark}
Writing the function with Representer theorem, after~\cref{remark:n=1}, we have $f(x)=\alpha_1k(x,x_1)$, where $x_1$ is a training sample. For clarity without loss of generality, hereafter, we let $\alpha_1=1$ and denote our single training sample with $x_t$.
\end{remark}

We denote the spatial dimension by $\tau$, along which we are to use a Fourier basis. Hence, assuming the training sample is a continuous function along the spatial dimension, we denote it by  $x_t(\tau)$.

Let $x_e$ denote the sample for which we seek a gradient-based explanation. Let $x_e'(\tau)$ denote the gradient of the kernel \wrt the spatial dimension $\tau$:
\begin{equation}
    x_e'(\tau) = \nabla_\tau k(x_e(\tau),x_t(\tau))
\end{equation}

\begin{remark}
We assume the input data exhibits a high degree of spatial input feature correlation both for $x_t(\tau)$ and $x_e(\tau)$. This can be simply expressed by a high concentration of spatial power spectrum of the input $|\widehat{x}(\tau)|^2$ around zero. Therefore, we assume there is a $\tau_0$ such that the following condition holds for all $\tau$ and $\tau'$:
\begin{equation}
     |\tau|<\tau_0 \text{ and } |\tau'|>\tau_0 \rightarrow |\widehat{x}(\tau)|^2\gg |\widehat{x}(\tau')|^2
\end{equation}
\label{remark:high-autocorr}
\end{remark}
This assumption is common for image data, where neighboring pixels tend to be highly correlated, see~\cref{sec:theoretically-unproven-assumptions} for further discussion.


Furthermore, a condition about $x_t$ and $x_e$, is the existence of an intersection at a certain point in spatial domain.
Let $\Delta(\tau) \coloneqq x_t(\tau)-x_e(\tau)$, then we can express this condition compactly as:
\begin{equation} 
\exists \tau^* \quad \text{such that} \quad \Delta(\tau^*) = 0,
\label{eq:condition-for-tail}
\end{equation}
As we will show in the next lemma, this intersection influences the behavior of the kernel's spectral decay. 

\begin{lemma}
    Let \( k: \mathbb{R}^d \times \mathbb{R}^d \to \mathbb{R} \) be a shift-invariant kernel with spatial Fourier transform \( \widehat{k}(\omega) \), and let \( \nabla_x k(x, x') \) denote its input-gradient. Suppose there exist trajectories \( x_t(\tau) \) and \( x_e(\tau) \), with high autocorrelation, as stated in~\cref{remark:high-autocorr}. 
    The asymptotic decay rate of the spatial power spectrum of the gradient kernel is primarily determined by the intersection condition stated in~\cref{eq:condition-for-tail}.
    \label{lemma:3}
\end{lemma}
\begin{proof}  
We want to take the Fourier transform (with respect to the spatial variable \(\tau\)) of the derivative  
\[
x_e'(\tau) = \frac{d}{dx} k\bigl(x_e(\tau), x_t(\tau)\bigr),
\]
where, by the shift-invariance of \(k\), we may express  
\[
\frac{d}{dx} k\bigl(x_e(\tau), x_t(\tau)\bigr) = \frac{d}{dx} k\bigl(x_e(\tau)-x_t(\tau)\bigr).
\]
To make the Fourier analysis tractable, we approximate \(\Delta(\tau)\) by a linear function in the spatial domain. Two cases arise:
\begin{enumerate}
    \item \textbf{No Intersection:} If there is no \(\tau\) for which \(\Delta(\tau)=0\), then a linear approximation yields \(\Delta(\tau)=\alpha\) with \(\alpha~\neq~0\). In this degenerate case, no root is present.
    
    \item \textbf{Intersection(s) Exist:} There exists at least one \(\tau^*\) satisfying  
   \[
   \Delta(\tau^*) = 0.
   \]
   In a neighborhood of such a point, by a first-order Taylor expansion,  
   \[
   \Delta(\tau) \approx \alpha (\tau-\tau^*),
   \]
   for some \(\alpha\neq0\) and \(\tau\) close to \(\tau^*\). This approximation ensures that the linearized \(\Delta\) has a root at \(\tau=\tau^*\), capturing the sharp transition in the kernel.
\end{enumerate}
In the typical setting with many training samples, it is reasonable to assume that such intersections exits, therefore we focus on case 2.
According to the definition of $\Delta$, we have
\[
\frac{d}{dx} k\bigl(x_e(\tau)-x_t(\tau)\bigr) = \frac{d}{dx} k\bigl(\Delta(\tau)\bigr).
\]
Under the linearization of $\Delta$, we can use a local change of variables \(x = \alpha \tau\) (which, is only valid locally due to the linear approximation), we can write
\[
\frac{d}{dx} k\bigl(-\Delta(\tau)\bigr) = \frac{1}{\alpha}\frac{d}{d\tau} k\bigl(\alpha(\tau-\tau^*)\bigr).
\]

Taking the Fourier transform with respect to \(\tau\) then yields
\[
\F_\tau\left\{x_e'(\tau)\right\} = \frac{1}{\alpha}\F_\tau\left\{\frac{d}{d\tau} k\bigl(\alpha(\tau-\tau^*)\bigr)\right\}.
\]
By the standard property of the Fourier transform, namely that differentiation corresponds to multiplication by \(i\omega\), it follows that
\[
\F_\tau\left\{\frac{d}{d\tau} k\bigl(\alpha(\tau-\tau^*)\bigr)\right\} = i\omega\,\F_\tau\left\{ k\bigl(\alpha(\tau-\tau^*)\bigr)\right\}.
\]
Thus, taking the magnitude squared, to compute the power spectrum, we obtain
\[
\left|\F_\tau\left\{\frac{d}{d\tau} k\bigl(\alpha(\tau-\tau^*)\bigr)\right\}\right|^2 = \omega^2 \left|\F_\tau\left\{ k\bigl(\alpha(\tau-\tau^*)\bigr)\right\}\right|^2.
\]
Hence, 
\[
\left|\F_\tau\left\{x_e'(\tau)\right\}\right|^2 = \frac{1}{\alpha^2}\,\omega^2 \left|\F_\tau\left\{ k\bigl(\alpha(\tau-\tau^*)\bigr)\right\}\right|^2.
\]
Due to the translation invariance of the Fourier transform, the shift by \(\tau^*\) does not alter the decay properties, so that
\[
\left|\F_\tau\left\{ k\bigl(\alpha(\tau-\tau^*)\bigr)\right\}\right|^2 = \left|\F_\tau\left\{ k\bigl(\alpha\tau\bigr)\right\}\right|^2.
\]
Recalling that the spatial Fourier transform of \(k\) is \(\hat{k}(\omega)\), we deduce
\[
\left|\F_\tau\left\{x_e'(\tau)\right\}\right|^2 = \frac{\omega^2}{\alpha^2}\left|\hat{k}(\omega)\right|^2,
\]
or, equivalently,
\[
\left|\F_\tau\left\{x_e'(\tau)\right\}\right|^2 \in \mathcal{O}\Bigl(\omega^2\,\hat{k}(\omega)^2\Bigr).
\]

Finally, while one might consider the possibility of multiple intersections (i.e., multiple neighborhoods where \(\Delta(\tau)\) changes sign), these contribute only as a multiplicative factor in the intermediate expressions (analogous to summing over intersections) and do not affect the order of decay rate. Therefore, the presence of at least one intersection governs the tail behavior of the spatial power spectrum of the gradient kernel. 
\end{proof}

Thus, we have demonstrated that the gradient of the samples can be approximated using local linear projections of the gradient of the kernel into the spatial domain. Consequently, the sharp transitions in the spatial domain are a direct consequence of sharp transitions in the gradient of the kernel. Since the tail behavior of the power spectrum depends only on the existence of such sharp transitions and not on their number, we considered a single intersection for a single sample in our analysis.

We now present the proof of Theorem 1 under the assumption that the training and explanation trajectories intersect. 
This assumption is made primarily for theoretical convenience.
In practice, factors beyond our theoretical model—such as random initialization—can induce sharp transitions at arbitrary locations. 
Consequently, the absence of such intersections is highly improbable when the explained sample lies within the support of the training data distribution. 
This also aligns with previous findings that sharp transitions induced by ReLU breakpoints, which introduce nonlinearity, occur not only on the training data~\cite{novak_sensitivity_2018}, but also in surrounding regions.


\begin{theorem}
    Let \(\mathcal{X}\subset\mathbb{R}^d\) be a fixed dataset and let \(f:\mathbb{R}^d\to\mathbb{R}\) be a neural network whose associated Neural Tangent Kernel (NTK) is denoted by \(K^{\text{(NTK)}}(c)\). Then, the asymptotic decay of the power spectrum of \(K^{\text{(NTK)}}(c)\) is directly proportional to the asymptotic decay of the power spectrum of the spatial Fourier transform of \(\nabla f(x)\).
    \label{thm:ntk-spectrum}
\end{theorem}

\begin{proof}
    By Lemma~\ref{lemma:2}, the spectral decay of the input-gradient \(\nabla f(x)\) of a shift-invariant kernel is characterized by
    \[
    \left|\mathcal{F}\{\nabla f\}\right|^2 = \mathcal{O}(n\,\omega^2|\hat{k}(\omega)|^2),
    \]
    when focusing on a single training sample and a corresponding explanation instance we realize that $n$ would be a constant.

    Next, under the high-autocorrelation assumption and the existence of an intersection between the training and explanation trajectories (as specified in~\cref{eq:condition-for-tail}), Lemma~\ref{lemma:3} shows that the tail behavior of the spatial power spectrum of the gradient kernel is governed by the local behavior at this intersection. In this region, a linear approximation of the difference \(\Delta(\tau)\) is valid, and the induced local (approximate) change of variables
    \[
    x = \alpha\, \tau
    \]
    (with \(\alpha\neq 0\)) enables us to relate the derivative with respect to \(x\) to that with respect to \(\tau\) via \(\frac{d}{dx} = \frac{1}{\alpha}\frac{d}{d\tau}\).

    Consequently, the Fourier transform of the gradient undergoes the transformation
    \begin{align}
    \mathcal{F}_\tau\left\{\frac{d}{dx} k(\cdot)\right\} &= \frac{1}{\alpha} \, i\omega\, \mathcal{F}_\tau\left\{ k(\alpha(\tau-c)) \right\}\\
        &=\frac{\omega^2}{\alpha^2}\left|\hat{k}(\omega)\right|^2
    \end{align}
    which implies that the power spectrum is asymptotically proportional to
    \[
    \mathcal{F}_\tau\left\{\frac{d}{dx} k(\cdot)\right\}\in \mathcal{O}\Bigl(\omega^2\,|\hat{k}(\omega)|^2\Bigr).
    \]
    
    Thus, the asymptotic decay of the power spectrum of \(K^{\text{(NTK)}}(c)\) is directly proportional to that of the spatial Fourier transform of \(\nabla f(x)\), as claimed.
\end{proof}


\paragraph{Derivations for \cref{eq:delta-ef}}
\label{sec:derivations-delta-ef}
In this section, we compute the two integrals related to the Neural Tangent Kernel (NTK) connections, under the assumption that the underlying kernel is the Laplace kernel~\cite{geifman_similarity_2020}. Since the domain under consideration is finite, we evaluate the integrals over a bounded interval $(l,h)$.

Our goal is to establish the relation $\mathcal{G}(f,\widehat{f})\sim\Delta \EF$, by analyzing the asymptotic behavior of both quantities with respect to the kernel variance parameter $b$, as it gets modified according to~\cref{lemma:1}.

From \cref{thm:1}, we know that the quantity $\S_{e_{f}}$ in \cref{eq:ef} is asymptotically equivalent to the power spectral density of the Laplace kernel:

\begin{equation}
    |\widehat{k}(\omega)|^2=\frac{2b}{1+b^2\omega^2}
\end{equation}
Moreover, \cref{lemma:2} shows that $\mathcal{G}(f,\widehat{f})$ in \cref{eq:delta-ef} has the same leading-order behavior.

Substituting into the integrals and analyzing their asymptotics yields
\begin{equation}
    \mathcal{G}(f,\widehat{f})\sim \Delta \EF \sim \mathcal{O}\left(\frac{1}{b}\right),
\end{equation}
as $b\rightarrow\infty$, and
\begin{equation}
    \mathcal{G}(f,\widehat{f})\sim \Delta \EF \sim \mathcal{O}(b),
\end{equation}
as $b\rightarrow 0^+$, thus confirming the scaling relation of interest.

We note that alternative derivations are possible, but we chose the most direct approach, which also ensures that the two orthogonal components $(\EF+\Delta\EF)$ share consistent units.

\section{Contribution of ReLU to NTK's Sharpness}
\label{sec:relu-vs-Softplus}

The literature of NTK is somewhat denser in terms of the results around ReLU networks, and as far as the authors are concerned, SoftPlus has not been considered as an option when analyzing properties of NTK.
Here, we introduce a creative technique to bypass tedious steps for finding analytical solutions to for NTK, assuming that we have some initial results, which is based of convolution operation defined as.
\begin{equation}
    f*g(u) = \int_{v\in\R} f(u-v)g(v) dv
\end{equation}
We focus on a work that provides a simple introduction into the equations needed to be solved for the ReLU NTK.

Summarizing one of the results in~\cite{simon_reverse_2022} for 1 hidden-layer neural networks with initialization weights and bias variances $\sigma_w^2=1$ , $\sigma^2_b=0$, we have to compute the $\tau$-transform, defined as:
\begin{equation}
    \tau_\phi\left(c ; p\right) = \mathbb{E}_{z_1, z_2 \sim p(z_1, z_2;c,\sigma^2)}\left[\phi\left(z_1\right) \phi\left(z_2\right)\right]
\end{equation}
where $\phi$ denotes the activation function, and 
\begin{equation}
    p(z_1,z_2;c,\sigma^2) = \mathcal{N}\left(\left[\begin{array}{l}
0 \\
0
\end{array}\right],\left[\begin{array}{cc}
\sigma^2 & c \sigma^2 \\
c \sigma^2 & \sigma^2
\end{array}\right]\right).
\end{equation}

After computing the $\tau$-transform, we can compute the NTK using the following recursive equations:
\begin{equation}
K^{(\mathrm{NTK})}(c) = K^{(0)}(c) +cK^{(1)}(c)
\end{equation}
where $K^{(0)}(c)$ and $K^{(1)}(c)$ are defined as follows
\begin{align}
K^{(0)}(c) & = \tau_\phi\left(c ; p\right) \\
K^{(1)}(c) & = \tau_{\phi^{\prime}}\left(c ; p\right)
\end{align}
where $\tau_{\phi^{\prime}}$ is defined as follows
$$
\tau_{\phi^{\prime}}\left(c ; p\right)=\frac{\partial_c}{\sigma^2} \tau_\phi\left(c ; p\right)
$$

\setcounter{lemma}{0}
\begin{lemma}
Let \(\phi: \mathbb{R} \to \mathbb{R}\) be an activation function, and let \(K^{(\text{NTK})}(c)\) denote the Neural Tangent Kernel (NTK) associated with \(\phi\). 
Define \(\phi_{\beta} = \phi * g_{\beta}\) as the activation function obtained by convolving \(\phi\) with a Gaussian function \(g_{\beta}(x) = \sqrt{\frac{\beta}{2\pi}} e^{-x^2\beta/2}\) of precision \(\beta\). 
Then, the NTK corresponding to \(\phi_{\beta}\), denoted as \(K^{(\text{NTK})}_{\beta}(c)\), leads to a smoother function in the sense that it exhibits faster decay, compared to \(K^{(\text{NTK})}(c)\).
\end{lemma}
\begin{proof}
We can simply start with the definition of $\tau$-transform and the convolution of a Gaussian function $g$ with ReLU $\phi$, as follows
\begin{align}
    \tau_{\phi_{\beta}}\left(c ; p\right)&=\tau_{\relu*g_{\beta}}\left(c ; p\right)\\
    &= \iint \phi_{\beta}(z_1)\phi_{\beta}(z_2)p(z_1,z_2) dz_1 dz_2\\
    &= \iint \relu*g_{\beta}(z_2) \relu*g_{\beta} (z_1)\nonumber\\
    &p(z_1,z_2) dz_1 dz_2\\
    &= \iiiint \relu(z_2-\nu_2) \relu(z_1-\nu_1)\nonumber\\
    &g_{\beta} (\nu_1)g_{\beta}(\nu_2) p(z_1,z_2) d\nu_1d\nu_2dz_1 dz_2\\
    &= \iiiint \relu(\kappa_1) \relu(\kappa_2) g_{\beta}(\nu_2) g_{\beta} (\nu_1)\nonumber\\ 
    &p(\kappa_1+\nu_1,\kappa_2+\nu_2) d\kappa_1d\kappa_2d\nu_1d\nu_2\\
    &= \iint \relu(\kappa_1) \relu(\kappa_2) \nonumber\\ 
    &\left(\iint g_{\beta}(\nu_2) g_{\beta} (\nu_1)p(\kappa_1+\nu_1,\kappa_2+\nu_2)d\nu_1d\nu_2\right)\nonumber\\ &d\kappa_1d\kappa_2\\
    &= \iint \relu(\kappa_1) \relu(\kappa_2) q(\kappa_1,\kappa_2)d\kappa_1d\kappa_2\\
    &= \tau_{\relu}\left(c ; q\right)
\end{align}
This shows that the convolution of ReLU with a Gaussian function changes the covariance matrix of the $\tau$-transform to:
\begin{equation}
    q(\boldsymbol{\kappa}) = \mathcal{N}\left(\left[\begin{array}{l}
0 \\
0
\end{array}\right],\left[\begin{array}{cc}
\sigma^2+\beta & c \sigma^2 \\
c \sigma^2 & \sigma^2 +\beta
\end{array}\right]\right).
\end{equation}
As the range of $c\in[-1,1]$ is fixed, increasing $\beta$ would lead to a matrix closer to identity, hence a smoother kernel.

We would like to conclude the proof by highlighting the fact that
\begin{equation}
    \frac{d}{dz}\left(\relu*g_{\beta}(z)\right) = \relu'*g_{\beta}(z)
\end{equation}
where $\relu' = \frac{d}{dz}\relu$. Therefore, the $\tau$-transform applied to the second term of the kernel, \ie $K^{(1)}$, would lead to the same derivations, with a consistent replacement of $\phi\rightarrow\phi'$.
\end{proof}

It is important to note that in practice, we approximate the convolution of Gaussian and ReLU with SoftPlus activation function. That is
\begin{equation}
    \phi_{\beta}(x,\beta) = \relu*g_{\beta}(x)\approx\text{SoftPlus}(x;\beta),
\end{equation}
with a proper choice of precision $\beta$, which is directly proportional to the parameter of SoftPlus. 

We have also verified the statement in Lemma~\ref{lemma:1} experimentally in~\cref{fig:smoother-kernel-gaussian-conv}, using the method in~\cite{murray_characterizing_2023} for SoftPlus activation function.
\begin{figure}[t]
\centering
   \includegraphics[width=\linewidth]{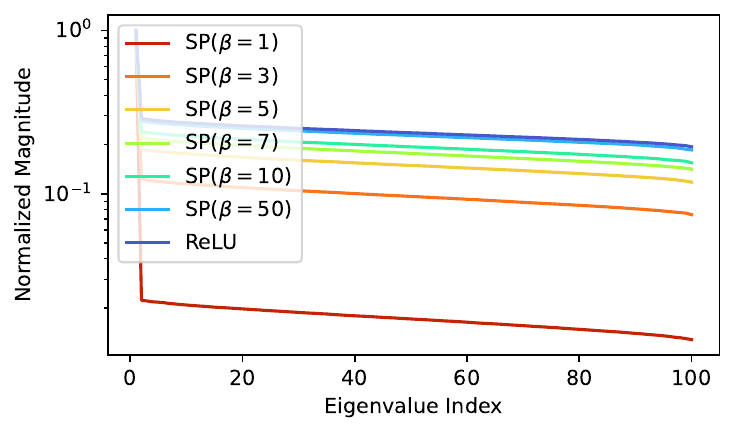}
   \caption{
   \textbf{Spectral Decay of the Empirical NTK Across Smooth Parameterizations of ReLU.}
   This figure depicts the spectral decay of the empirical neural tangent kernel NTK, plotting the eigenvalues (x-axis) against their normalized magnitudes (y-axis). This empirical evaluation supports Lemma~\ref{lemma:1}, demonstrating how the choice of activation function influences the spectral tail behavior. 
   Notably, ReLU exhibits the heaviest tail, while increasing the Softplus $\beta$ parameter, shown by SP($\beta$), results in a sharper decay. 
   For further insights, compare these results with~\cref{fig:psd-tail-inputsizes} where there is a very similar progression in the tail of the spatial power spectrum. This figure is produced by the code available in \cite{murray_characterizing_2023}.}
\label{fig:smoother-kernel-gaussian-conv}
\end{figure}

\section{Regarding Input's Spatial Autocorrelation}
\label{sec:theoretically-unproven-assumptions}

\begin{figure*}[t]
\centering
\begin{subfigure}[t]{0.24\linewidth}
    \centering
    \includegraphics[width=\linewidth]{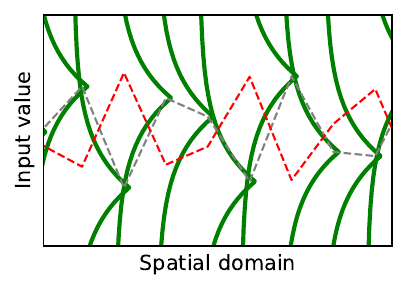}
    \caption{}
\end{subfigure}
\begin{subfigure}[t]{0.24\linewidth}
    \centering
    \includegraphics[width=\linewidth]{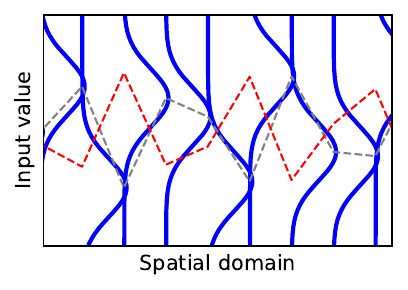}
    \caption{}
\end{subfigure}
\begin{subfigure}[t]{0.24\linewidth}
    \centering
    \includegraphics[width=\linewidth]{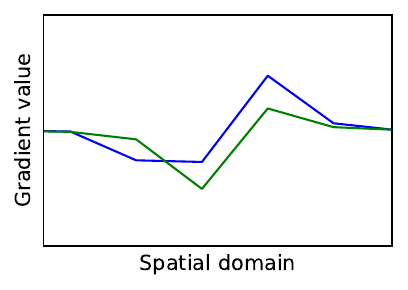}
    \caption{}
\end{subfigure}
\begin{subfigure}[t]{0.24\linewidth}
    \centering
    \includegraphics[width=\linewidth]{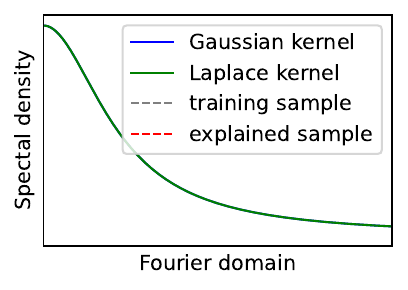}
    \caption{}
\end{subfigure}
   \caption{
   \textbf{Impact of Low Spatial Autocorrelation on Spatial Power Spectrum of Gradient.}
   This figure illustrates the impact of low spatial autocorrelation in the input and, consequently, in the learned features on the tail behavior of the spatial spectral density of the gradient. Compare with~\cref{fig:kernel-shape-and-sharpness}.
\textbf{(a)} and \textbf{(b)}: As in~\cref{fig:kernel-shape-and-sharpness}, the x-axis represents spatial dimensions, while the y-axis represents pixel values. 
However, unlike~\cref{fig:kernel-shape-and-sharpness}, where the input and learned features exhibit high spatial autocorrelation, here both display low spatial autocorrelation. 
This phenomenon can be observed in various data modalities, such as data frames. 
As a result, the relationship between input samples—specifically, the sample to be explained (red line) and the training sample (gray line)—appears more irregular.
\textbf{(c)}: Taking the gradient of the classifier with respect to the input yields a function in the spatial domain. Compared to~\cref{fig:kernel-shape-and-sharpness}, the gradient function here is less structured, reflecting the reduced spatial autocorrelation.
\textbf{(d)}: In this case, applying the Fourier transform to the gradient values may not reveal distinct spectral decay rates for the Laplace and Gaussian kernels, as the tail behavior is also influenced by the input autocorrelation.
}
\label{fig:low-feature-autocorrelation}
\end{figure*}

There are important confounding factors in our analysis, regarding the autocorrelation of the learned features and that of the input features of the data.
If the input data has high spatial autocorrelation, then we naturally expect a high autocorrelation in the learned features  after training.

Even though being intuitive, specially in image data, this phenomenon has not been theoretically proven yet.

Unfortunately, proving this property is outside the scope of explainability and is more towards the literature around kernel methods or feature alignment.
This is important as we assume continuity and high spatial autocorrelation in our analysis for recovering the behavior of the tail of the power spectrum, without a formal proof.

We are aware that such assumption would rule out the existence and influence of sharp changes in the image, such as edges.

For a visualization of low feature autocorrelation that might correspond to a network after initialization, or a network trained on a data with low spatial autocorrelation, compare~\cref{fig:kernel-shape-and-sharpness,fig:low-feature-autocorrelation}.
\setlength{\bsize}{0.11\textwidth}
\begin{figure*}[t]
\centering
\begin{subfigure}[t]{\bsize}
    \centering
    \includegraphics[width=\linewidth]{content/banner/404_image.pdf}
\end{subfigure}
\begin{subfigure}[t]{\bsize}
    \centering
    \includegraphics[width=\linewidth]{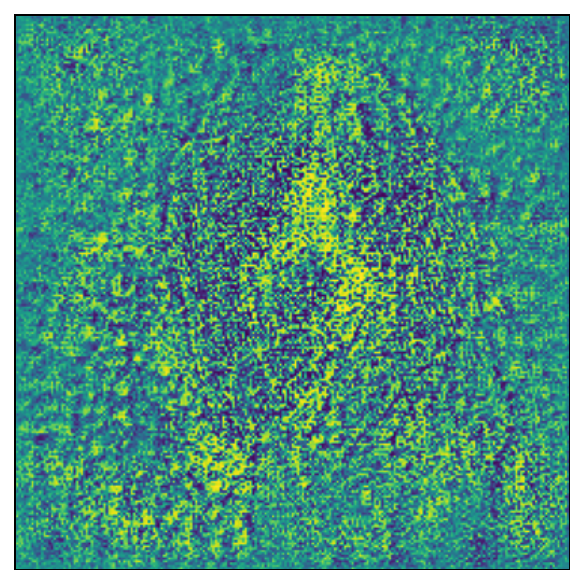}
\end{subfigure}
\begin{subfigure}[t]{\bsize}
    \centering
    \includegraphics[width=\linewidth]{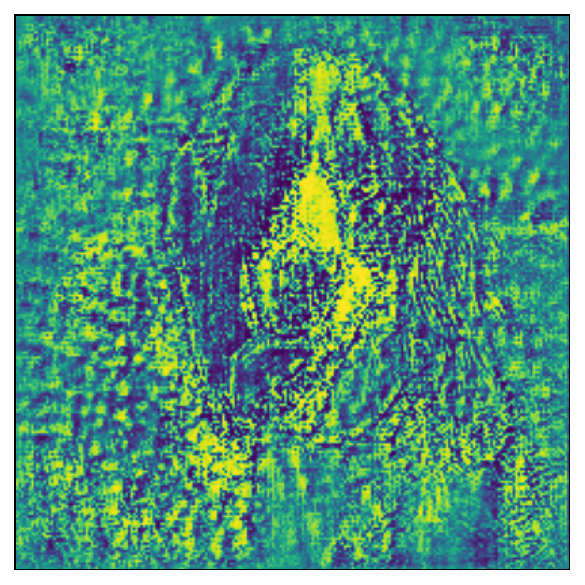}
\end{subfigure}
\begin{subfigure}[t]{\bsize}
    \centering
    \includegraphics[width=\linewidth]{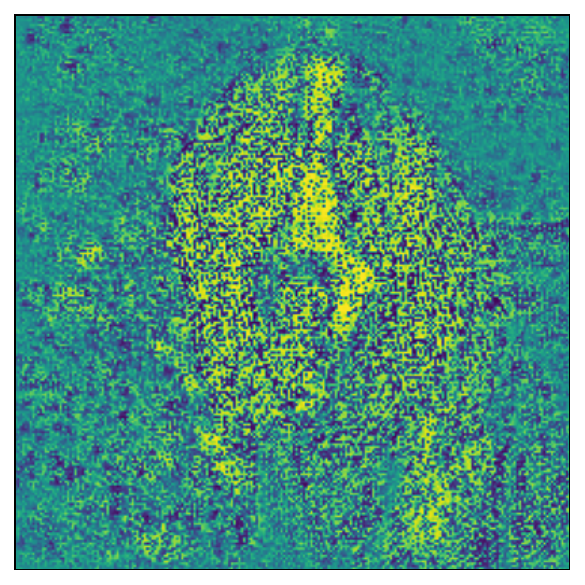}
\end{subfigure}
\begin{subfigure}[t]{\bsize}
    \centering
    \includegraphics[width=\linewidth]{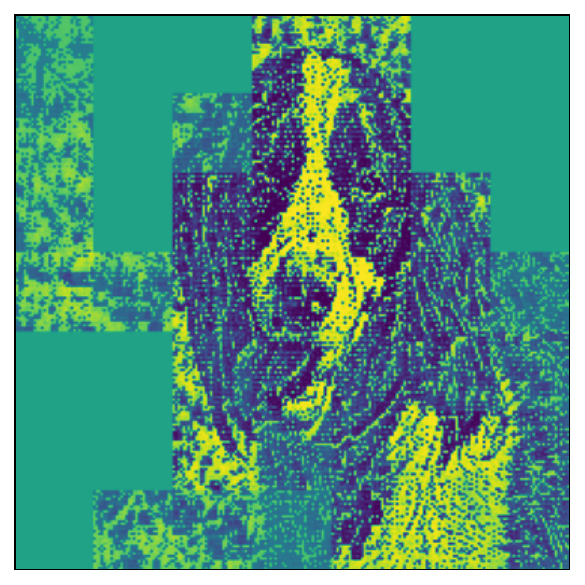}
\end{subfigure} 
\begin{subfigure}[t]{\bsize}
    \centering
    \includegraphics[width=\linewidth]{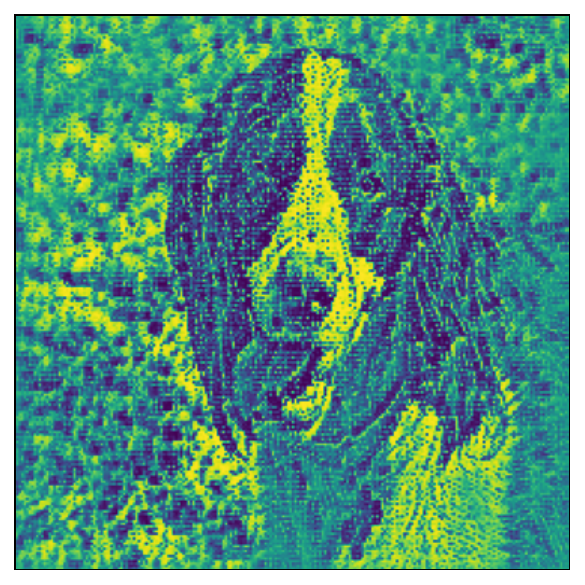}
\end{subfigure} 
\begin{subfigure}[t]{\bsize}
    \centering
    \includegraphics[width=\linewidth]{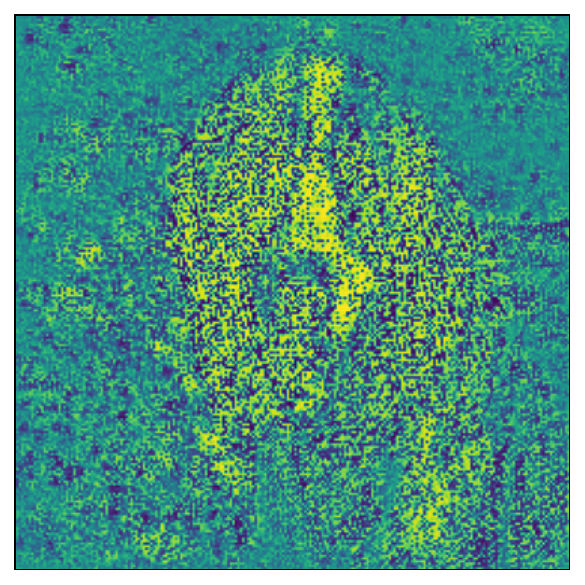}
\end{subfigure}
\begin{subfigure}[t]{\bsize}
    \centering
    \includegraphics[width=\linewidth]{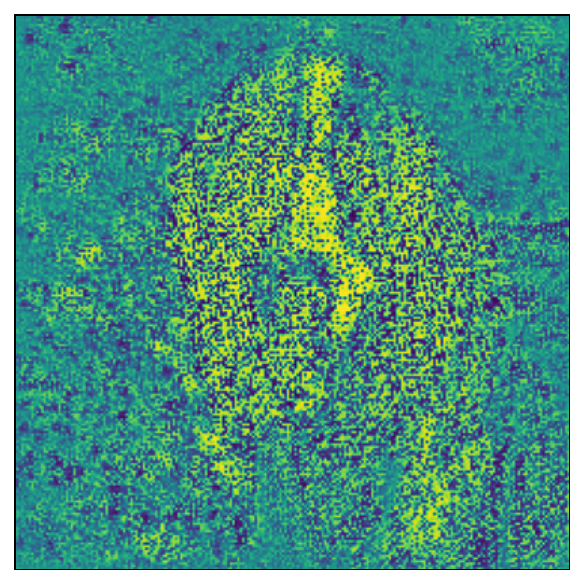}
\end{subfigure}

\begin{subfigure}[t]{\bsize}
    \centering
    \includegraphics[width=\linewidth]{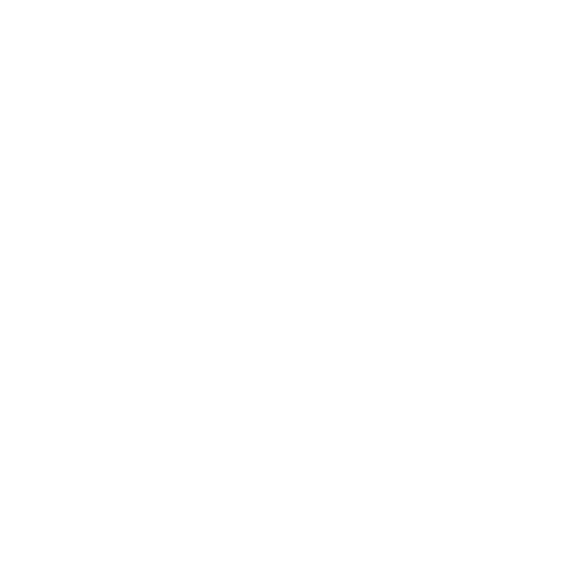}
    \caption{}
\end{subfigure}
\begin{subfigure}[t]{\bsize}
    \centering
    \includegraphics[width=\linewidth]{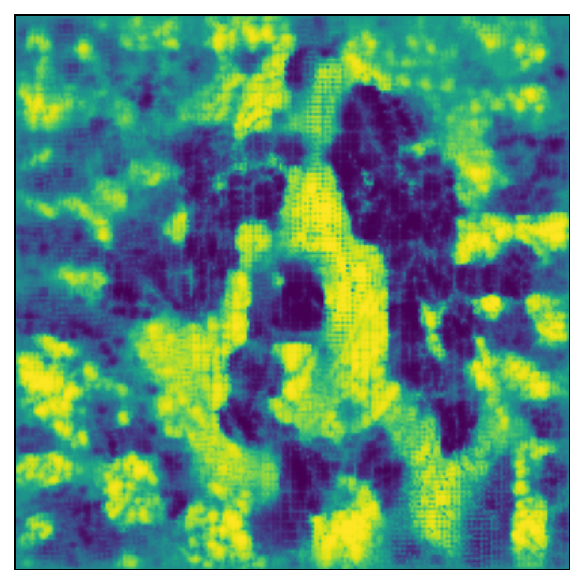}
    \caption{}
\end{subfigure}
\begin{subfigure}[t]{\bsize}
    \centering
    \includegraphics[width=\linewidth]{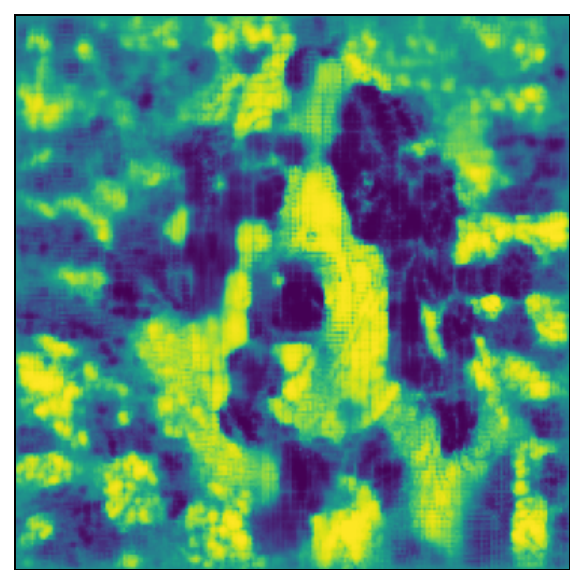}
    \caption{}
\end{subfigure}
\begin{subfigure}[t]{\bsize}
    \centering
    \includegraphics[width=\linewidth]{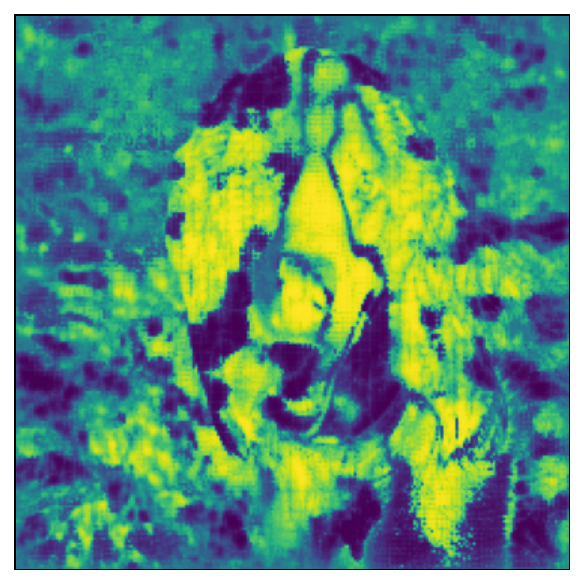}
    \caption{}
\end{subfigure}
\begin{subfigure}[t]{\bsize}
    \centering
    \includegraphics[width=\linewidth]{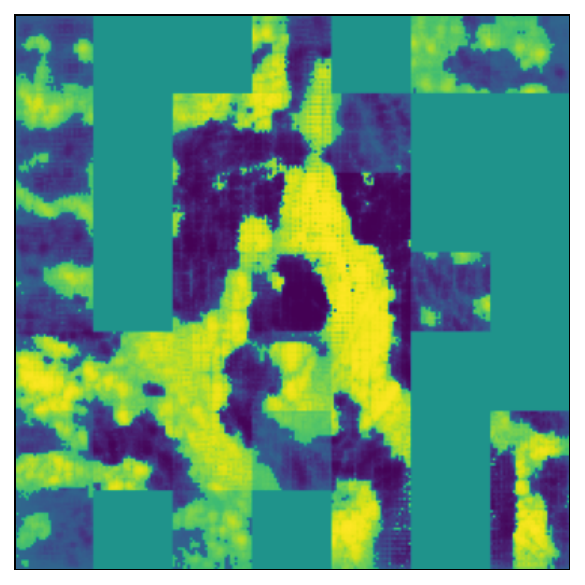}
    \caption{}
\end{subfigure} 
\begin{subfigure}[t]{\bsize}
    \centering
    \includegraphics[width=\linewidth]{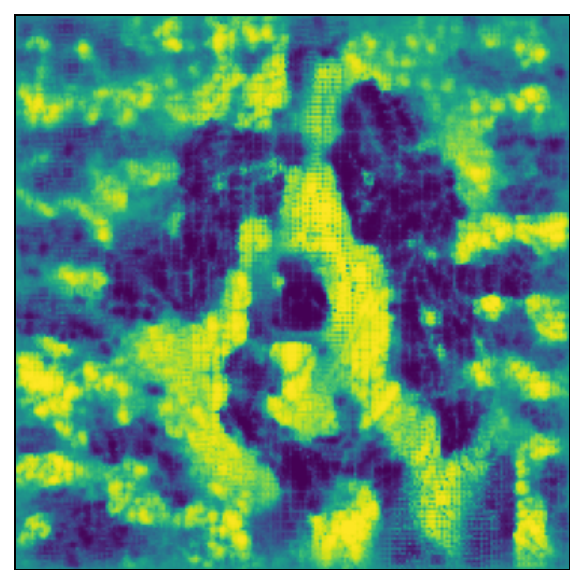}
    \caption{}
\end{subfigure} 
\begin{subfigure}[t]{\bsize}
    \centering
    \includegraphics[width=\linewidth]{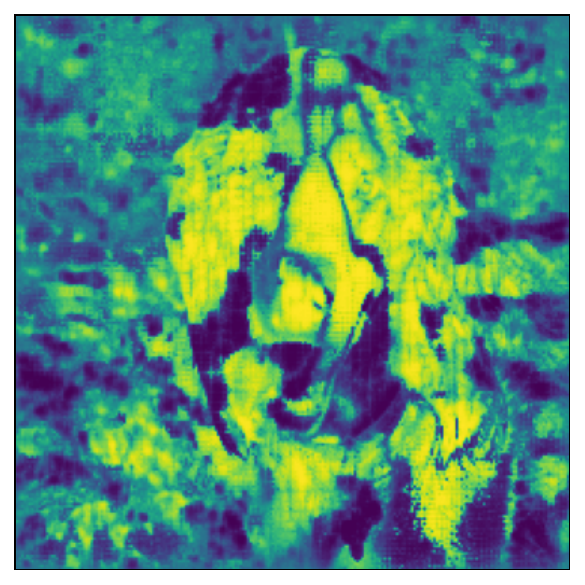}
    \caption{}
\end{subfigure}
\begin{subfigure}[t]{\bsize}
    \centering
    \includegraphics[width=\linewidth]{content/vis_explainer/white.pdf}
    \caption{}
\end{subfigure}
   \caption{
   \textbf{Sample Visualization of Explanations Outputs After Inverse Transformation Normalization.}
   This figure presents a visualization of the outputs from various explanation methods after applying the inverse transformation method. This normalization technique aligns different distributions while preserving their spectral properties and maintaining insensitivity to magnitude.
   }
\label{fig:vis-explainers-example}
\end{figure*}

\begin{figure*}[t]
\centering
\begin{subfigure}[t]{0.162\linewidth}
    \centering
    \includegraphics[width=\linewidth]{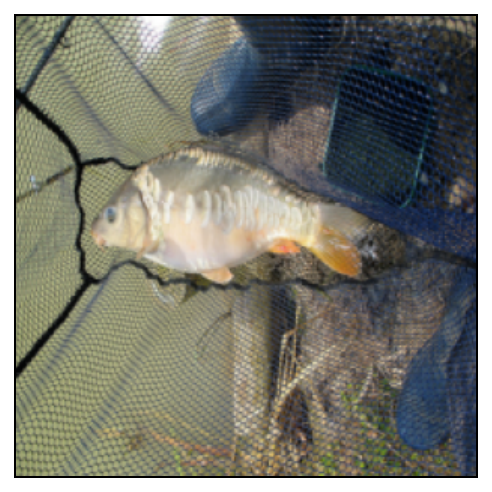}
\end{subfigure}
\begin{subfigure}[t]{0.162\linewidth}
    \centering
    \includegraphics[width=\linewidth]{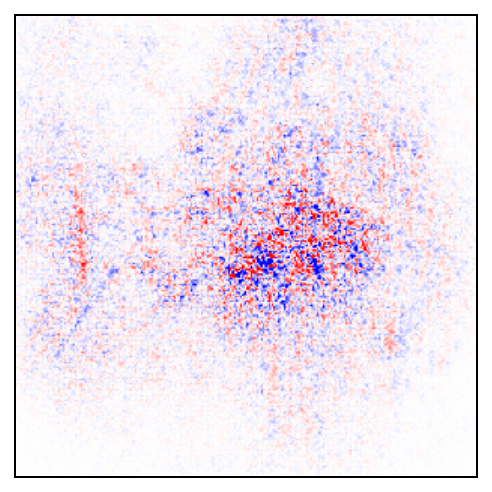}
\end{subfigure}
\begin{subfigure}[t]{0.162\linewidth}
    \centering
    \includegraphics[width=\linewidth]{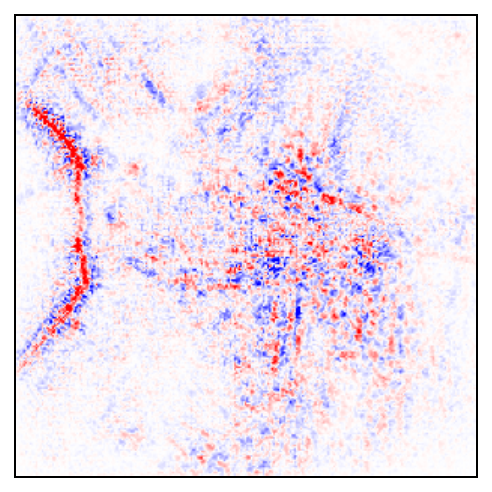}
\end{subfigure}
\begin{subfigure}[t]{0.162\linewidth}
    \centering
    \includegraphics[width=\linewidth]{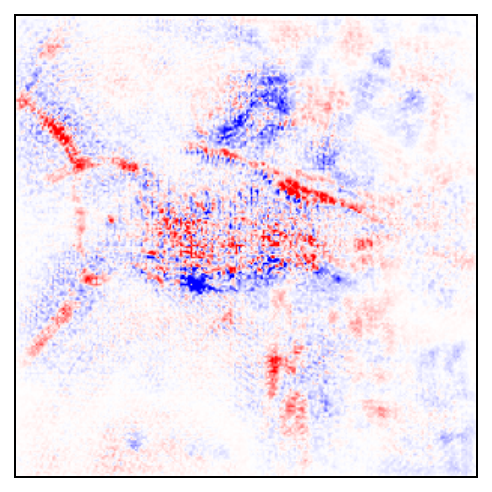}
\end{subfigure} 
\begin{subfigure}[t]{0.162\linewidth}
    \centering
    \includegraphics[width=\linewidth]{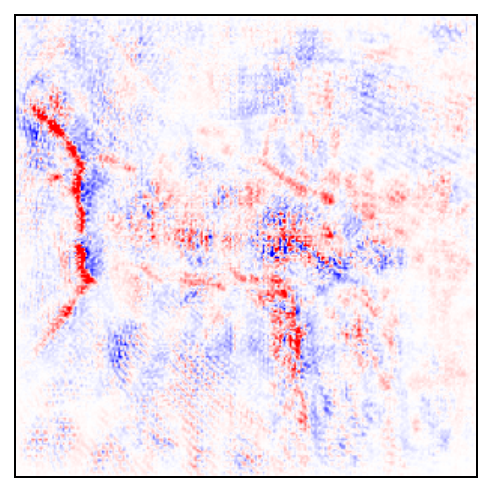}
\end{subfigure} 
\begin{subfigure}[t]{0.162\linewidth}
    \centering
    \includegraphics[width=\linewidth]{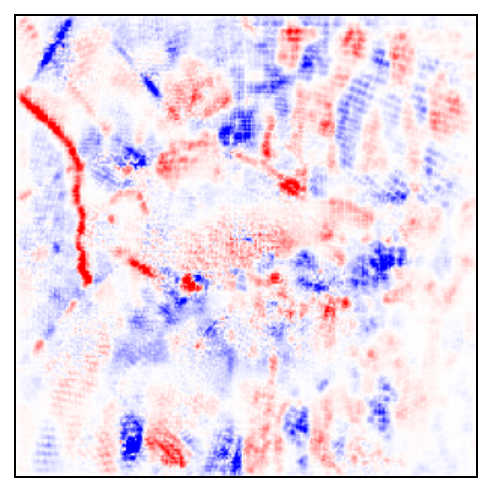}
\end{subfigure} 
\begin{subfigure}[t]{0.162\linewidth}
    \centering
    \includegraphics[width=\linewidth]{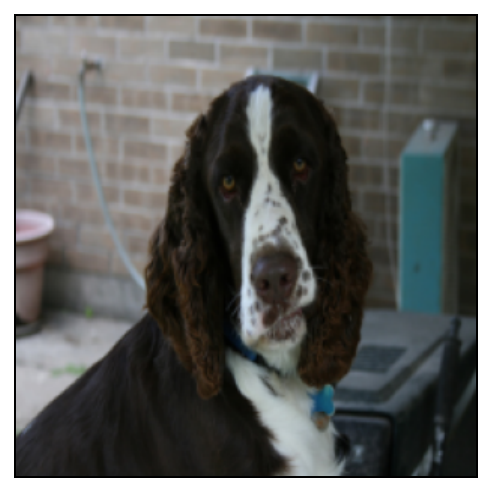}
\end{subfigure}
\begin{subfigure}[t]{0.162\linewidth}
    \centering
    \includegraphics[width=\linewidth]{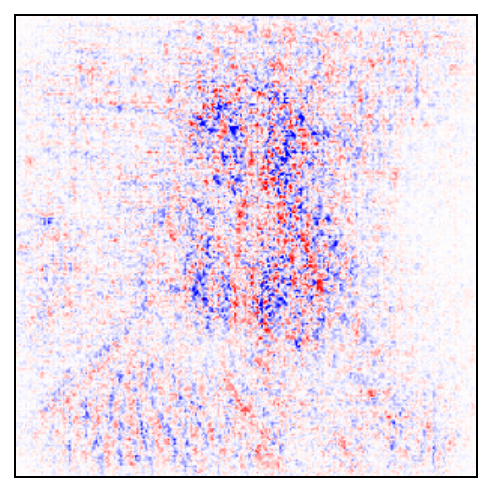}
\end{subfigure}
\begin{subfigure}[t]{0.162\linewidth}
    \centering
    \includegraphics[width=\linewidth]{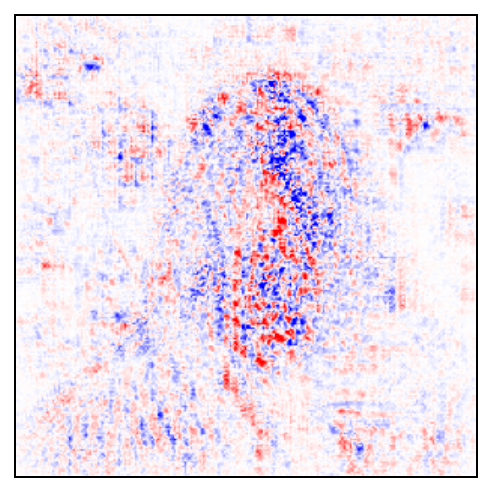}
\end{subfigure}
\begin{subfigure}[t]{0.162\linewidth}
    \centering
    \includegraphics[width=\linewidth]{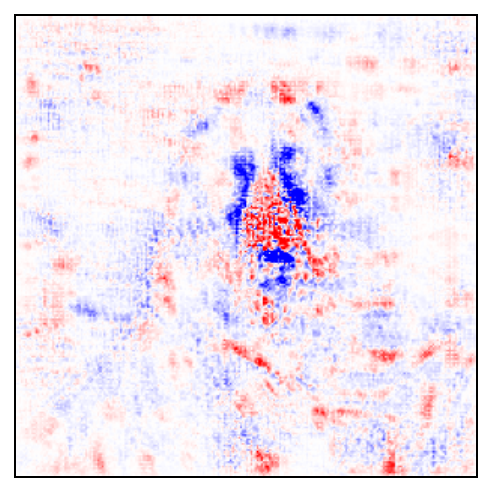}
\end{subfigure} 
\begin{subfigure}[t]{0.162\linewidth}
    \centering
    \includegraphics[width=\linewidth]{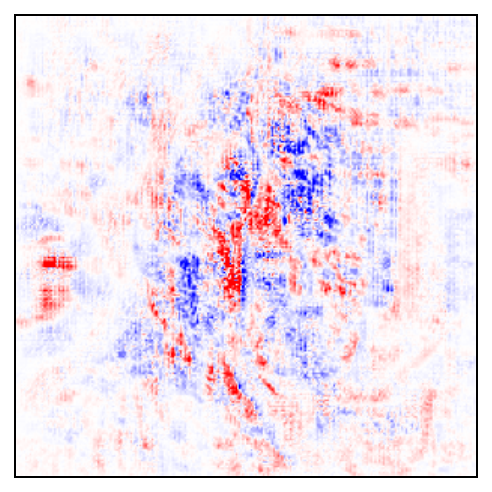}
\end{subfigure} 
\begin{subfigure}[t]{0.162\linewidth}
    \centering
    \includegraphics[width=\linewidth]{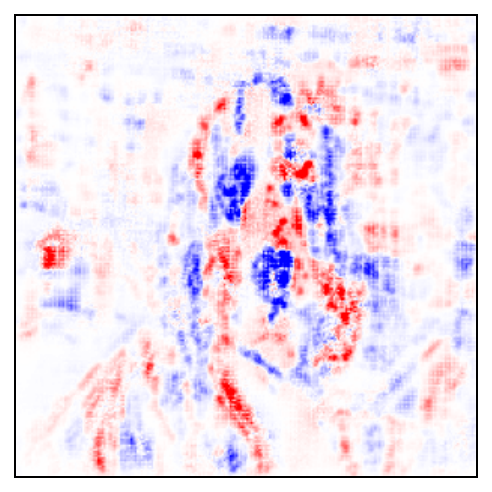}
\end{subfigure} 
\begin{subfigure}[t]{0.162\linewidth}
    \centering
    \includegraphics[width=\linewidth]{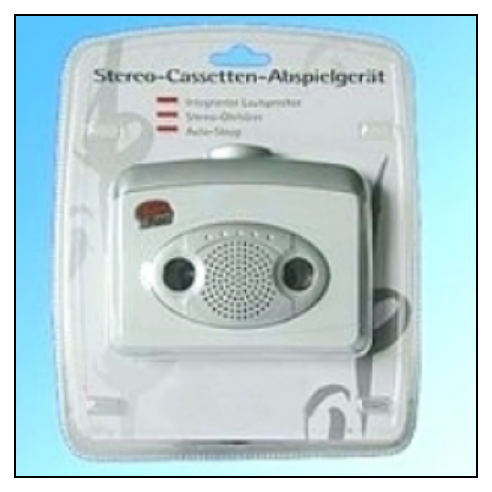}
\end{subfigure}
\begin{subfigure}[t]{0.162\linewidth}
    \centering
    \includegraphics[width=\linewidth]{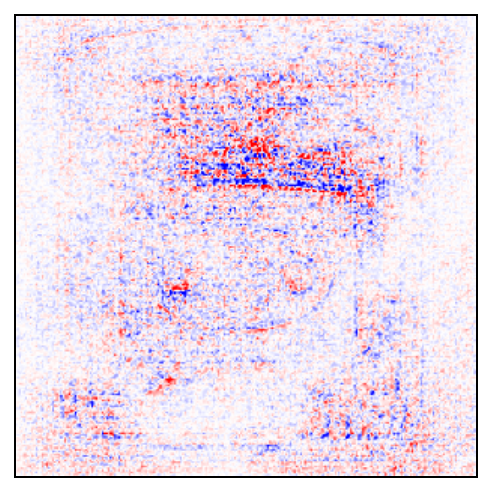}
\end{subfigure}
\begin{subfigure}[t]{0.162\linewidth}
    \centering
    \includegraphics[width=\linewidth]{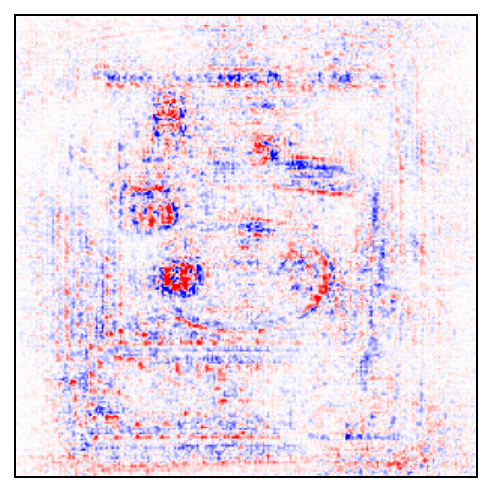}
\end{subfigure}
\begin{subfigure}[t]{0.162\linewidth}
    \centering
    \includegraphics[width=\linewidth]{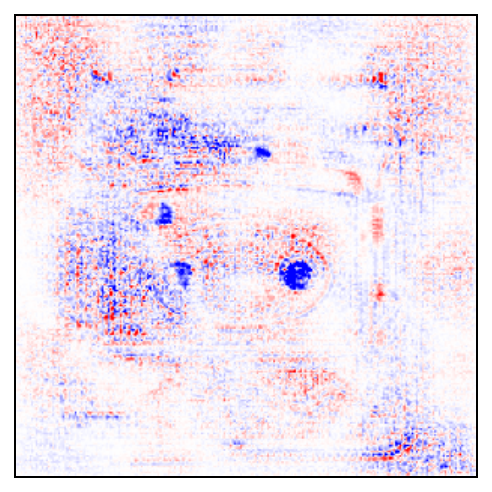}
\end{subfigure} 
\begin{subfigure}[t]{0.162\linewidth}
    \centering
    \includegraphics[width=\linewidth]{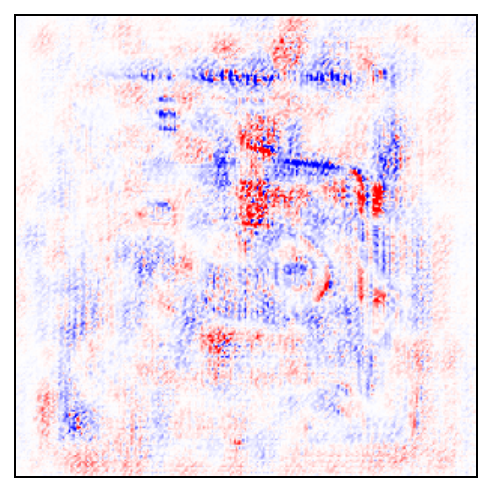}
\end{subfigure} 
\begin{subfigure}[t]{0.162\linewidth}
    \centering
    \includegraphics[width=\linewidth]{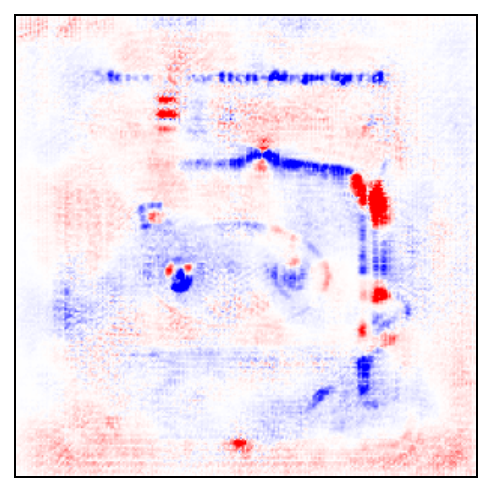}
\end{subfigure} 
\begin{subfigure}[t]{0.162\linewidth}
    \centering
    \includegraphics[width=\linewidth]{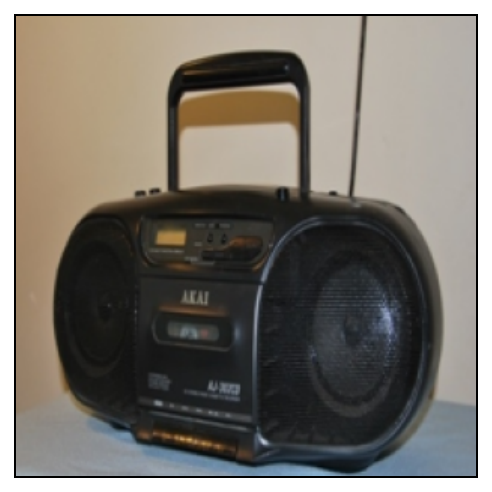}
    \caption{Input}
\end{subfigure}
\begin{subfigure}[t]{0.162\linewidth}
    \centering
    \includegraphics[width=\linewidth]{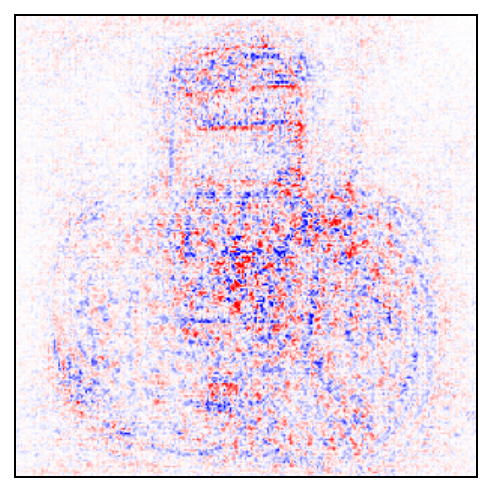}
    \caption{ReLU}
\end{subfigure}
\begin{subfigure}[t]{0.162\linewidth}
    \centering
    \includegraphics[width=\linewidth]{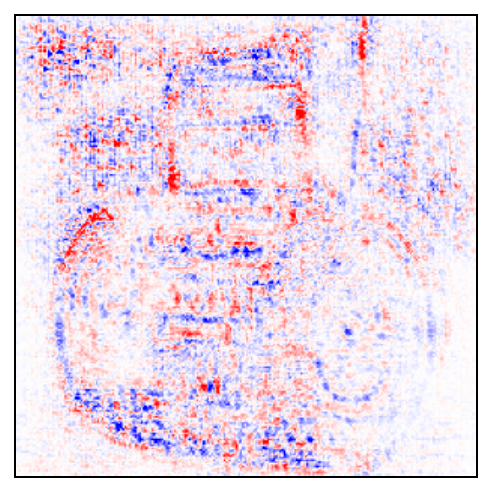}
    \caption{$\operatorname{SP(\beta=7)}$}
\end{subfigure}
\begin{subfigure}[t]{0.162\linewidth}
    \centering
    \includegraphics[width=\linewidth]{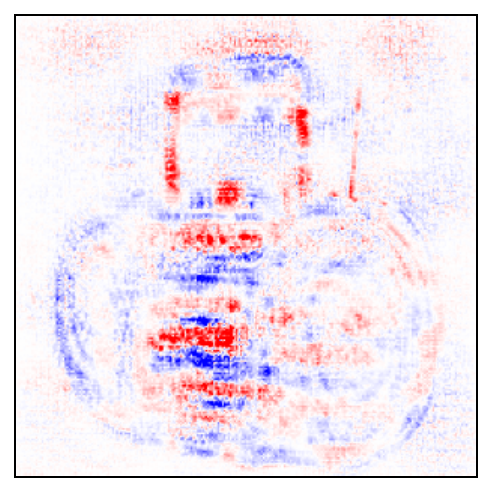}
    \caption{$\operatorname{SP(\beta=5)}$}
\end{subfigure} 
\begin{subfigure}[t]{0.162\linewidth}
    \centering
    \includegraphics[width=\linewidth]{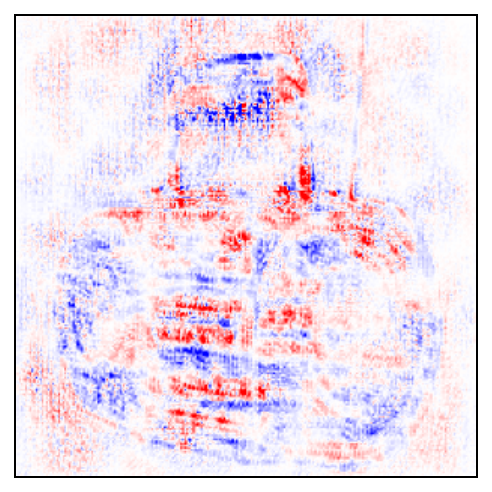}
    \caption{$\operatorname{SP(\beta=3)}$}
\end{subfigure} 
\begin{subfigure}[t]{0.162\linewidth}
    \centering
    \includegraphics[width=\linewidth]{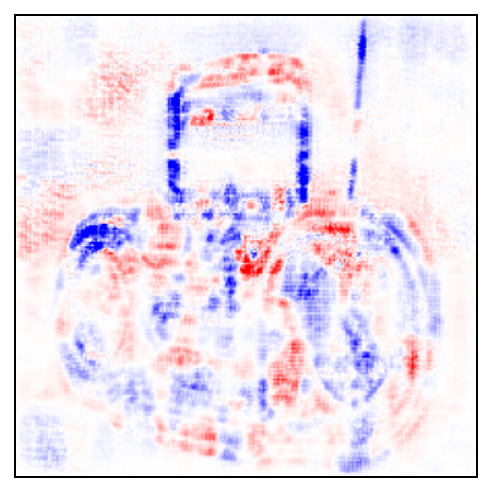}
    \caption{$\operatorname{SP(\beta=.9)}$}
\end{subfigure} 
   \caption{
   \textbf{Sample Visualization of VanillaGrad Explanations Across Smooth ReLU Parameterizations.}
   This figure presents additional samples from trained networks with different smooth parameterizations of ReLU, denoted by $\operatorname{SP}(\beta)$. 
   The goal is to illustrate how the reliance on high-frequency information in a ReLU network and its smooth variants manifests in VanillaGrad explanations.  
   This phenomenon has been discussed in more detail in~\cref{sec:controling-the-tps}.
   VanillaGrad explanations provide a local snapshot of the network's dependence on features across different frequencies, also known as harmonics~\cite{kadkhodaie_generalization_2024,marchetti_harmonics_2024}. Some of these harmonics can be "suppressed" through surrogates introduced by explanation methods.  
   To ensure an unaltered view of this effect, all visualizations in this figure utilize the simplest gradient-based explanation method, VanillaGrad, applied \textit{without any surrogates}.
   }
\label{fig:more-examples}
\end{figure*}

\end{document}